\documentclass{article}

\usepackage{microtype}
\usepackage{graphicx}
\usepackage{subcaption}
\usepackage{booktabs} 

\usepackage{hyperref}

\usepackage[preprint]{icml2026}

\usepackage{amsmath}
\usepackage{amssymb}
\usepackage{mathtools}
\usepackage{amsthm}

\usepackage[capitalize,noabbrev]{cleveref}

\theoremstyle{plain}
\newtheorem{theorem}{Theorem}[section]
\newtheorem{proposition}[theorem]{Proposition}

\theoremstyle{definition}
\newtheorem{definition}[theorem]{Definition}

\theoremstyle{remark}

\usepackage[textsize=tiny]{todonotes}

\usepackage{xcolor}
\usepackage{amssymb}
\usepackage{amsthm}

\definecolor{darkblue}{rgb}{0, 0.12, 0.55}
\definecolor{darkgreen}{rgb}{0, 0.55, 0.12}
\definecolor{darkred}{rgb}{0.6,0,0}
\definecolor{darkgreen}{rgb}{0,0.6,0}
\definecolor{Gray}{gray}{0.9}
\definecolor{mymauve}{rgb}{0,0.6,0}
\definecolor{mygreen}{rgb}{0,0.6,0}
\definecolor{maincolor}{rgb}{0.2,0.5,0.7}
\definecolor{myblue}{rgb}{0.2549, 0.4118, 0.8824}
\definecolor{daoqi}{rgb}{0.1176, 0.5647, 1.0}
\definecolor{haijun}{rgb}{0.27, 0.51, 0.71}
\usepackage{url}
\usepackage{wrapfig}
\usepackage{graphicx}
\usepackage{multirow}
\usepackage{booktabs}
\usepackage{dsfont}
\usepackage[table]{xcolor}
\usepackage{amsmath}
\usepackage{graphicx}
\usepackage{makecell}
\usepackage{tcolorbox}
\usepackage{graphicx}
\usepackage{subcaption} 
\usepackage{caption} 

\newtcolorbox{promptbox}{
    colback=gray!10,        
    colframe=black,         
    sharp corners,          
    boxrule=0.8pt,          
    left=6pt,               
    right=6pt,
    top=6pt,
    bottom=6pt
}

\icmltitlerunning{}

\newcommand{\ie}{\emph{i.e., }}
\newcommand{\eg}{\emph{e.g., }}

\newcommand{\nosection}[1]{\vspace{2pt}\noindent\textbf{#1.}}

\newcommand{\modelname}{\texttt{ODEdit}}

\newcommand{\revise}[1]{\textcolor{black}{#1}}

\begin{document}

\twocolumn[
  \icmltitle{Generalizable Multimodal Large Language Model Editing \\
via Invariant Trajectory Learning}



  \icmlsetsymbol{equal}{*}

  \begin{icmlauthorlist}
    \icmlauthor{Jiajie Su}{yyy}
    \icmlauthor{Haoyuan Wang}{yyy}
    \icmlauthor{Xiaohua Feng}{yyy}
    \icmlauthor{Yunshan Ma}{sch}
    \icmlauthor{Xiaobo Xia}{comp}\\
    \icmlauthor{Yuyuan Li}{zzz}
    \icmlauthor{Xiaolin Zheng}{yyy}
    \icmlauthor{Jianmao Xiao}{kkk}
    \icmlauthor{Chaochao Chen}{yyy,equal}
    
  \end{icmlauthorlist}

  \icmlaffiliation{yyy}{Zhejiang University, China}
   \icmlaffiliation{sch}{Singapore Management University, Singapore}
  \icmlaffiliation{comp}{National University of Singapore, Singapore}
  \icmlaffiliation{zzz}{Hangzhou Dianzi University, China}
  \icmlaffiliation{kkk}{Jiangxi Normal University, China}

  \icmlcorrespondingauthor{Chaochao Chen}{zjuccc@zju.edu.cn}

  \icmlkeywords{Machine Learning, ICML}

  \vskip 0.3in
]



\printAffiliationsAndNotice{}  

\begin{abstract}
Knowledge editing emerges as a crucial technique for efficiently correcting incorrect or outdated knowledge in large language models (LLMs).
Existing editing methods rely on a rigid mapping from parameter or module modifications to output, which causes the generalization limitation in Multimodal LLM (MLLM).
In this paper, we reformulate MLLM editing as an out-of-distribution (OOD) generalization problem, where the goal is to discern semantic shift with factual shift and thus achieve robust editing among diverse cross-modal prompting.
The key challenge of this OOD problem lies in identifying invariant causal trajectories that generalize accurately while suppressing spurious correlations.
To address it, we propose \modelname,a plug-and-play invariant learning based framework that optimizes the tripartite OOD risk objective to simultaneously enhance editing reliability, locality, and generality.
We further introduce an edit trajectory invariant learning method, which integrates a total variation penalty into the risk minimization objective to stabilize edit trajectories against environmental variations.
Theoretical analysis and extensive experiments demonstrate the effectiveness of \modelname.
%
\end{abstract}

\section{Introduction}

With rapid applications of large language models (LLMs) \citep{liu2024deepseek}, ensuring their knowledge correctness and currency in a cost-efficient manner has become a critical concern.
\textit{Knowledge editing} \citep{de2021editing,wang2023easyedit,wang2024knowledge} is an emerging technique that supports data-efficient modifications on pre-trained models within a specific scope of knowledge.
%
%
Existing LLM editing methods can be broadly categorized into: i) \textbf{parameter-adjusting} \citep{meng2022mass,fang2024alphaedit,jiang2025anyedit}, which directly tune a subset of parameters in the original model, and ii) \textbf{model-extending} methods \citep{huang2023transformer,hartvigsen2023aging,yu2024melo}, which attach auxiliary components while keeping the backbone parameters intact.
However, current studies remain largely confined to unimodal LLMs, leaving open their extension to multimodal LLMs (MLLMs) \citep{du2025mmke,guo2025balancedit}.

\begin{figure}[t]
\centering
\includegraphics[width=1.0\linewidth]{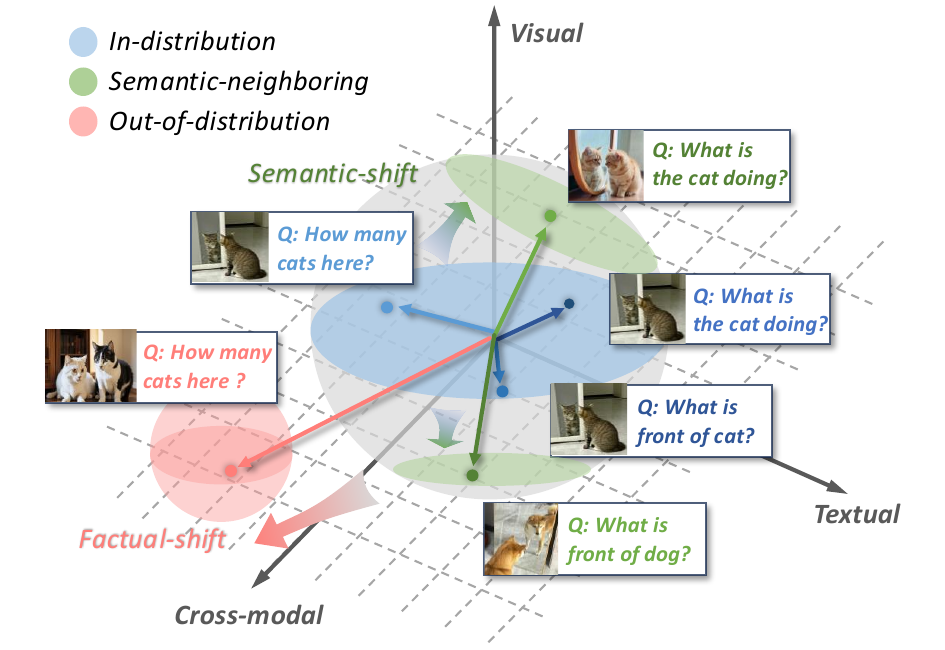}
\vspace{-0.25 in}
\caption{The motivation of \modelname. We illustrate how semantic-shift and factual-shift emerge in this editing OOD problem. We give out the formal definition of two shifts in Appendix \ref{app:shiftdefine}.} 
\label{story}
\vspace{-0.25 in}
\end{figure}

Several pioneer studies have made initial attempts to edit MLLMs.
UniKE \cite{pan2024towards} performs both parameter-adjusting and model-extending, by representing intrinsic and external knowledge as separate vectorized key-value memories.
BalancEdit \cite{guo2025balancedit} as a model-extending approach, generates positive and negative samples to determine the influence scope, and edits with a localized codebook in the latent space.
Despite advances in promoting edit success (\textit{reliability}), these works all operationalize editing as \textbf{a rigid mapping} from parameter modifications $\Delta W$ or auxiliary component $\Delta M$ to output variations $\Delta Y$, distilled from limited training samples.
Nevertheless, within MLLM, the output variation is shaped by cascaded reasoning as \textit{unimodal perception → cross-modal alignment → shared semantic space reasoning}, such that localized modification to parameters or modules can trigger complex and non-local semantic propagation effects.
Relying on such a rigid mapping thus imposes the \textbf{generalization problem} from two aspects:
i) \textit{First}, it fails to disentangle the coherent \textit{shared} semantic structures that span diverse cross-modal contexts.
This precludes MLLM from discovering reusable semantic substrates, reducing editing to a brittle case-wise alignment exercise rather than principled semantic \textbf{generality}.
ii) \textit{Second}, it drives MLLM to memorize entrenched linkages between local features and outputs.
The inflexible mapping misguides the model when faced with queries that share local features yet differ in specific causal semantics, thus compromising \textbf{locality}.

To address it, we rethink MLLM editing as an out-of-distribution (OOD) cross-modal semantic generalization problem.
%
OOD \citep{ye2021towards,montasser2024transformation} refers to identifying invariant versus spurious features that drive distribution shifts, enabling the model to generalize to unseen domains.
%
%
In Figure \ref{story}, editing MLLM involves partitioning cross-modal semantic scopes in prompts, which has two distributional shifts:
\revise{(i) \textit{Semantic-shift} indicates the shift from in-distribution editing scopes to neighboring regions, constituting generalized targets.
%
}
Once the editing instills the \textit{mirror reflection} principle into MLLM, the model can generalize from a narrowly edited instance (\textit{a tabby cat staring itself in the mirror}) to similar scenarios (\textit{an orange cat or dog doing the same thing}).
\revise{
(ii) \textit{Factual-shift} denotes the transition from in-distribution to out-of-distribution regions, encompassing extraneous concepts beyond editing scopes.}
%
%
The mirror reflection principle should not be overapplied to counting prompts when lacking mirror-specific visual features.
Building on these, robust MLLM editing requires identifying \textbf{invariant trajectories} for cross-modal predictions and removing \textbf{spurious factors} that disrupt semantical reasoning associations.
%



In this paper, we propose a plug-and-play editing OOD optimization framework for multimodal LLM, termed \modelname, which leverages cross-modal reasoning trajectory invariant learning to ensure knowledge editing robustness across diverse distributions.
%
%
To explicitly enhance the MLLM’s \textit{discriminative awareness of semantic-shift and factual-shift}, we first introduce a tripartite OOD risk formulation that imposes tailored constraints on in-distribution, semantic-neighboring, and out-of-distribution features.
%
%
We apply the Kullback-Leibler divergence regularization to preserve locality while developing a maximum mean discrepancy-based metric learning to align representations of edited concepts and their semantic variants.
To \textit{discern and stabilize the edit trajectories} across heterogeneous cross-modal environments, we further propose Edit Trajectory Invariant Learning (ETIL).
ETIL first reforms the editing OOD objective into an equivalent invariant risk minimization problem, where an environment-aware classifier is introduced to exploit feature invariance and irrelevance.
Then, to suppress the sensitivity of the edit trajectory to spurious environmental changes, ETIL integrates a Total Variation factor as the penalty term in the risk estimation.
The invariant risk minimization is achieved through a primal-dual optimization strategy, ensuring that the edited model captures reusable causal structures while filtering out superficial correlations.

Main contributions are
(1) We revisit MLLM editing from the OOD generalization perspective, and propose a plug-and-play optimization paradigm.
(2) We introduce a tripartite OOD risk that imposes tailored constraints on semantic- and factual-shift, and develop a trajectory invariant learning to minimize composed editing risk across diverse cross-modal prompting.
(3) We provide theoretical analyses and extensive experiments to validate effectiveness of \modelname.

\section{Related Work}

\nosection{Unimodal LLM Editing}
Model editing aims to modify the target knowledge in LLM while preserving irrelevant concepts.
Previous approaches can be divided into two types.
\textit{Parameter-adjusting methods} modify intrinsic parameters of LLMs to update new knowledge.
In this line, locate-then-edit models such as ROME \citep{meng2022locating}, MEMIT \citep{meng2022mass}, GLAME \citep{zhang2024knowledge},  AnyEdit \citep{jiang2025anyedit}, AlphaEdit \citep{fang2024alphaedit}, first identify crucial knowledge-related parameters and then perform targeted edits.
%
%
Besides, meta-learning based approaches like KE \citep{de2021editing}, MEND \citep{mitchell2021fast}, InstructEdit \citep{zhang2024instructedit}, determine parameter modifications by training hypernetworks. 
Contrastingly, \textit{model-extending methods} incorporate additional components to store new knowledge while keeping original model parameters.
The added components take diverse forms, including memory in SERAC \citep{mitchell2022memory} and WISE \citep{wang2024wise}, auxiliary neurons in T-Patcher \citep{huang2023transformer}, codebooks in GRACE \citep{hartvigsen2023aging}, and LoRA modules in MELO \citep{yu2024melo}.
Other works like MemPrompt \citep{madaan2022memprompt}, IKE \citep{zheng2023can}, and DeCK \citep{bi2024decoding} utilize in-context learning to update factual knowledge.
Despite their efficacy in unimodal LLM editing, they suffer from causal-underfit and causal-overfit issues in MLLM.
%

\nosection{Multimodal LLM Editing}
Recent advances in MLLMs \citep{li2023blip,li2024seed,Ma_Chen_Zhang_Wu_Ding_2025,luo2025next} have motivated research on multimodal knowledge editing \citep{pan2023finding,zhoum2edit}.
A series of benchmarks, \eg MMEdit \citep{cheng2023can}, MIKE \citep{li2024mike}, VLKEB \citep{huang2024vlkeb}, MC-MKE \citep{zhang2024mc}, MMKE \citep{du2025mmke}, provide unified datasets and evaluation to assess multimodal editing efficacy.
However, research on strengthening the robustness of MLLM editing methods holistically across reliability, locality, and generality remains underexplored.
MSCKE \citep{zeng2024visual} establishes a classifier-based knowledge editor to identify and update specific visual entities.
UniKE \citep{pan2024towards} integrates intrinsic knowledge editing and external knowledge resorting to promote locality and generality.
\revise{BalancEdit \citep{guo2025balancedit} performs codebook-based edits that balance generality and locality.}
Nevertheless, existing work remains constrained to rigid parameter-to-output mappings, which prevent MLLMs from intelligently distinguishing between semantic-shift and factual-shift, thereby hindering adaptive and robust editing.

\nosection{Out-of-Distribution Generalization}
OOD generalization is a core challenge in machine learning, aiming for generalization under covariate shift without access to data in the target domain \citep{muandet2013domain,arjovsky2019invariant,zhou2025fedgog}.
%
The mainstream works \citep{arjovsky2019invariant,krueger2021out,ahuja2020invariant,lai2024invariant} utilize invariant risk minimization with regularizer to explore invariant representations across different training environments.
Further, a wide range of techniques is leveraged to extract and generalize invariant features \citep{yu2023distribution}, \eg context-based augmentation \citep{nam2021reducing}, representation alignment \citep{dou2019domain,ruan2021optimal}, gradient manipulation \citep{shahtalebi2021sand}, distributional robust optimization \citep{ghosal2023distributionally}, and meta-learning \citep{chen2023secure}.
In this paper, we make the first attempt to cast MLLM editing as an OOD generalization problem, where invariant learning across editing environments is enforced via a total invariance regularizer on cross-modal semantic features, so as to improve editing robustness and adaptability.


\section{Preliminary}
\nosection{Out-of-Distribution Generalization}
Considering datasets $\mathcal{D}_{\text{e}}:=\{\mathbf{x}_i^{\text{e}},\mathbf{y}_i^{\text{e}}\}^{n_{\text{e}}}_{i=1}$ collected from diverse training environments ${\text{e}}\in \mathcal{E}_{train}$, the environments correspond to identical random variables assessed under distinct conditions.
The dataset $\mathcal{D}_{\text{e}}$ consists of i.i.d. samples drawn from the probability distribution $\mathcal{P}(\mathbf{X}^{\text{e}},\mathbf{Y}^{\text{e}})$.
OOD generalization targets at learning a predictor $f: \mathcal{X}\rightarrow\mathcal{Y}$ that minimizes the worst-case risk over a broad, potentially unseen set of environments $\mathcal{E}_{\text{all}} \supseteq \mathcal{E}_{\text{train}}$:
\begin{equation}
\small
    \mathcal{R}_{\text{OOD}}(f) = \max_{{\text{e}} \in \mathcal{E}_{\text{all}}} \mathbb{E}_{(\mathbf{X}^{\text{e}}, \mathbf{Y}^{\text{e}}) \sim \mathcal{P}^{\text{e}}} [\ell(f(\mathbf{X}^{\text{e}}), \mathbf{Y}^{\text{e}})].
    \nonumber
\end{equation}
Here, $\mathbb{E}_{(\mathbf{X}^{\text{e}}, \mathbf{Y}^{\text{e}}) \sim \mathrm{P}^{\text{e}}} [\ell(f(\mathbf{X}^{\text{e}}), \mathbf{Y}^{\text{e}})]$ denotes the risk under specific environment $e$, and $\ell$ is a suitable loss function. 
$\mathcal{E}_{\text{all}}$ includes environments not encountered during training.

\nosection{Invariant Risk Minimization (IRM)}
IRM \citep{arjovsky2019invariant,tan2023provably} generalizes invariant features to different environments.
Given training data as $\mathcal{D}:=\{(\mathbf{x}_i,\mathbf{y}_i)\in \mathcal{X} \times \mathcal{Y}\}$ where $\mathcal{X}$ and $\mathcal{Y}$ denotes the input and output space.
IRM constructs the learning model $\mathcal{X} \rightarrow \mathcal{Y}$ into two parts, \ie the feature extractor $\Psi: \mathcal{X} \rightarrow \mathcal{H}$ mapping input into the invariant feature space and the classifier $\omega: \mathcal{H} \rightarrow \mathcal{Y}$ predicting based on these features.
The empirical risk under environment $e$ is:
\begin{equation}
\small
    \mathcal{R}(\omega \circ \Psi, e) = \frac{1}{n} \sum_{i=1}^{n} \mathcal{L}(\omega \circ \Psi(\mathbf{x}_i), \mathbf{y}_i, e),
    \nonumber
\end{equation}
where $\mathcal{L}$ is the loss function.
The original IRM formulation is a bi-level optimization problem:
\begin{equation}
\small
\min_{\omega, \Psi} \sum_{e \in \mathcal{E}_{tr}} \mathcal{R}(\omega \circ \Psi, e)  \quad\text{s.t. } \omega \in \operatorname*{arg\,min}_{\tilde{w}} \mathcal{R}(\tilde{\omega} \circ \Psi, e), \forall e \in \mathcal{E}.
\nonumber
\end{equation}
This constraint requires $\omega$ to be optimal for each environment given $\Psi$, encouraging $\Psi$ to extract invariant features.
Further, IRMv1 \citep{arjovsky2019invariant} provides a surrogate form, which fixes the classifier $\omega$ to a constant scalar and replaces the constraint with a gradient norm penalty:
\begin{equation}
    \min_{\Psi} \sum_{e \in \mathcal{E}} \left\{ \mathcal{R}(1 \circ \Psi, e) + \lambda \left\| \nabla_{\omega|_{\omega=1}} \mathcal{R}(\omega \circ \Psi, e) \right\|_2^2 \right\}.
    \nonumber
\end{equation}

\section{Methodology}


\subsection{Problem Setting}\label{sec:setting}

\nosection{MLLM Editing as OOD Problem}
First, we formulate the knowledge editing task in a MLLM with the out-of-distribution generalization form.
Considering the MLLM as a function $\mathcal{M}:\mathcal{I} \times \mathcal{X} \rightarrow \mathcal{Y}$ with parameters $\phi$, which takes the cross-modal prompt $(\mathbf{m}_{\text{e}}, \mathbf{x}_{\text{e}})$ consisting of an image $\mathbf{m}_e$ and a textual description $\mathbf{x}_e$ as input, and generates $\mathbf{y}_o$ as the original output.
Denote the editing dataset containing facts to be updated as $\mathcal{D}_{\text{edit}}$, we define an environment factor $e \in \mathcal{E}$ which parameterizes the data distribution $\mathcal{P}_{{\text{e}}}(\textbf{M},\textbf{X},\textbf{Y})$, indicating all the possible causal associations that can occur in testing prompts.
The objective of MLLM editing is to update $\phi \rightarrow \phi_{\text{e}}$ for the worst-case risk $\mathcal{R}_{\text{edit}}(\phi_{\text{e}}, e)$ across all conceivable environments:
\begin{equation}
\small
\label{eq:ood}
    \min_{\phi_{\text{e}}} \max_{e \in \mathcal{E}}\mathbb{E}_{(\mathbf{m}_{\text{e}},\mathbf{x}_{\text{e}},\mathbf{y}_{\text{e}})\sim \mathcal{P}_{{\text{e}}}(\textbf{M},\textbf{X},\textbf{Y})} \mathcal{R}_{\text{edit}}(\phi_{\text{e}},(\mathbf{m}_{\text{e}},\mathbf{x}_{\text{e}},\mathbf{y}_{\text{e}}),e).
\nonumber
\end{equation}
According to existing benchmarks \cite{cheng2023can,huang2024vlkeb,du2025mmke}, the training and testing prompts $\mathcal{D}_{\text{test}}$ are composed of in-distribution data $\mathcal{D}_{\text{in}}$, semantic-neighboring data $\mathcal{D}_{\text{se}}$, and out-of-distribution data $\mathcal{D}_{\text{out}}$.
The overall risk is defined as $\mathcal{R}_{\text{edit}}$, \ie the composite measure of three editing metrics: $\mathcal{R}_{\text{rel}}$, $\mathcal{R}_{\text{gen}}$, and $\mathcal{R}_{\text{loc}}$, which respectively justify three aspects, \ie editing accuracy on $\mathcal{D}_{\text{IN}}$, side effects on $\mathcal{D}_{\text{OUT}}$, and generalization ability on $\mathcal{D}_{\text{SE}}$ :
\begin{equation}
\small
\begin{aligned}
    \mathcal{R}_{\text{rel}}:&= \mathbb{E}_{ (\mathbf{m}_{\text{e}},\mathbf{x}_{\text{e}},\mathbf{y}_{\text{e}})\sim\mathcal{P}_{\text{IN}}}\left[ \mathds{1}\{ \mathcal{M}(\mathbf{m}_{\text{e}},\mathbf{x}_{\text{e}};\phi_{\text{e}}),\mathbf{y}_{\text{e}})\}\right]\\
    \mathcal{R}_{\text{loc}}:&= \mathbb{E}_{ (\mathbf{m}_{\text{t}},\mathbf{x}_{\text{t}},\mathbf{y}_{\text{t}})\sim\mathcal{P}_{\text{OUT}}} \left[\mathds{1} \{\mathcal{M}(\mathbf{m}_{\text{t}},\mathbf{x}_{\text{t}};\phi_{\text{e}})=\mathcal{M}(\mathbf{m}_{\text{t}},\mathbf{x}_{\text{t}};\phi)\} \right]\\
    \mathcal{R}_{\text{gen}}:&= \mathbb{E}_{ (\mathbf{m}_{\text{r}},\mathbf{x}_{\text{r}},\mathbf{y}_{\text{r}})\sim\mathcal{P}_{\text{SE}}} \left[\mathds{1} \{\mathcal{M}(\mathbf{m}_{\text{e}},\mathbf{x}_{\text{e}};\phi_{\text{e}})=\mathcal{M}(\mathbf{m}_{\text{r}},\mathbf{x}_{\text{r}};\phi_{\text{e}}\} \right] 
\end{aligned}
\nonumber
\end{equation}

\subsection{Semantic-Factual Shift Disentanglement}\label{sec:alignment}

To facilitate MLLM discriminate editing environments between semantic-shift and factual-shift, we first design independent editing risks to evaluate transferability on invariant trajectories and capability to eliminate spurious factors.
Our framework aims to construct a unified optimization paradigm that is agnostic to specific editing methods, so it can be incorporated into any parameter-adjusting or model-extending editing approach based on fine-tuning.
With the pre-trained multimodal LLM $\mathcal{M}_{\phi}$ and editing dataset as $\mathcal{D}_{\text{edit}}$, we denote the editing model as $f_{\theta}$.
%
Editing is cast as learning a mapping $\Gamma$ that adapts the model and its parameters guided by the edit instance and $f_{\theta}$:
\begin{equation}
\small
    \mathcal{M}(\phi_{\text{e}}, \theta_{\text{e}})=\Gamma\left(\mathcal{M}_\phi, f_\theta ;\left(\mathbf{m}_{\text{e}}, \mathbf{x}_{\text{e}}, \mathbf{y}_{\text{e}}\right)\right)(\cdot), \left(\mathbf{m}_{\text{e}}, \mathbf{x}_{\text{e}}, \mathbf{y}_{\text{e}}\right)\sim \mathcal{P}_{\text{IN}}.
\end{equation}
Then, to optimize the three objectives outlined in Section \ref{sec:setting}, \ie reliability, locality, and generality, we propose corresponding risk metrics that are seamlessly integrated into these base editing models.

\nosection{Reliability Risk}
To ensure precise assimilation of the edited knowledge, we minimize the negative log-likelihood of the target output conditioned on the edit instance:
\begin{equation}
\small
    \mathcal{R}_{\text{rel}} = -\log p_{\phi_{\text{e}}}(\mathbf{y}_{\text{e}} \mid \mathbf{m}_{\text{e}}, \mathbf{x}_{\text{e}}), \quad \left(\mathbf{m}_{\text{e}}, \mathbf{x}_{\text{e}}, \mathbf{y}_{\text{e}}\right)\sim \mathcal{P}_{\text{IN}},
\end{equation}
which explicitly maximizes the probability of the desired output $\mathbf{y}_{\text{e}}$ for the edited input $(\mathbf{m}_{\text{e}}, \mathbf{x}_{\text{e}})$, ensuring accurate cognition on the in-distribution cross-modal semantics.

\nosection{Locality Risk}
In order to avoid the edited concepts affecting the interpretation of unrelated content falling within the factual-shift scope, we regularize the editing process by imposing a Kullback–Leibler divergence 
\citep{attias1999variational} penalty between the pre- and post-edit output distributions:
\begin{equation}
\scriptsize
    \mathcal{R}_{\text{loc}} = \mathrm{KL}\left(p_{\phi_{\text{e}}}\left(\cdot \mid \mathbf{m}_{\text{t}}, \mathbf{x}_{\text{t}}\right) \| p_\phi\left(\cdot \mid \mathbf{m}_{\text{t}}, \mathbf{x}_{\text{t}}\right)\right), \left(\mathbf{m}_{\text{t}}, \mathbf{x}_{\text{t}},\mathbf{y}_{\text{t}}\right)\sim \mathcal{P}_{\text{OUT}}.
\end{equation}
This constraint strengthens model capacity to preserve knowledge beyond the designated editing scope, minimizing unintended side effects.

\nosection{Generality Risk}
Previous methods \citep{mitchell2022memory,zeng2024visual} mostly emphasize supervised partitioning of in-scope and out-of-scope knowledge regions, but fall short in achieving semantic generalization, and thus cause issues of causal-underfit or causal-overfit.
Thus, we propose a generality risk for extracting invariant trajectories hidden underneath semantic-shift cross-modal prompting.
\revise{For each edited instance $(\mathbf{m}_{\text{e}},\mathbf{x}_{\text{e}})$, we utilize its rephrase counterparts $(\mathbf{m}_{\text{r}},\mathbf{x}_{\text{r}})$ from the benchmark training datasets.}
Let $\boldsymbol{z}_{\phi_e}(\mathbf{m}, \mathbf{x})$ denote the last hidden states of edited model $\mathcal{M}_{\phi_{\text{e}}}$ for prompt $(\mathbf{m}, \mathbf{x})$, we retrieve the distributions of edited prompts and rephrase prompts as $\boldsymbol{Z}_{E}$ and $\boldsymbol{Z}_{R}$ respectively.
Then we develop a Maximum Mean Discrepancy (MMD) \citep{tolstikhin2016minimax} based metric learning to measure the discrepancy between in- and semantic-neighboring distributions.
Given the Kernel Hilbert Space $\mathcal{H}$ associated with the Borel measurable kernel $k$, the mean embedding $\boldsymbol{\mu}_{\boldsymbol{Z}_{E}}$ and $\boldsymbol{\mu}_{\boldsymbol{Z}_{R}}$ is formulated with:
\begin{equation}
\small
\begin{aligned}
\boldsymbol{\mu}_{\boldsymbol{Z}_{E}}&=\int_{\mathbb{S}}k(s,\cdot)\boldsymbol{Z}_{E}(ds)\in \mathcal{H},\\
\boldsymbol{\mu}_{\boldsymbol{Z}_{R}}&=\int_{\mathbb{V}}k(v,\cdot)\boldsymbol{Z}_{R}(dv)\in \mathcal{H},
\end{aligned}
\end{equation}
where $s$ and $v$ are random variables with distribution $\boldsymbol{Z}_{E}$ and $\boldsymbol{Z}_{R}$.
It satisfies the distribution probability density equation that for all functions $f\in \mathcal{F}$:
\begin{equation}
\small
\begin{aligned}
    \mathbb{E}\left[f(S)\right]=\langle f, \boldsymbol{\mu}_{\boldsymbol{Z}_{E}}\rangle_{\mathcal{H}}, \quad 
    \mathbb{E}\left[f(V)\right]= \langle f, \boldsymbol{\mu}_{\boldsymbol{Z}_{R}}\rangle_{\mathcal{H}}.
\end{aligned}
\end{equation}
We deploy the multi-scale Gaussain kernel function $k(a_i, a_j) = \sum_{q=1}^{k} \exp\!\left(-\frac{\lVert a_i - a_j \rVert_2^2}{2\sigma_q^2}\right)$ in $\mathcal{H}$ to simultaneously capture local and global similarity between two instances, where $\sigma_q$ denotes the bandwidth of $q$-th kernel.
The generality risk in the MMD form is defined as:
\begin{equation}
\small
\begin{aligned}
    \mathcal{R}_{\text{gen}} =  &\mathbb{E}_{z_{\text{e}},z_{\text{e}}'\sim \boldsymbol{Z}_{E}}\left[k(\boldsymbol{z},\boldsymbol{z}_{\text{e}}')\right]
    +\mathbb{E}_{z_{\text{r}},z_{\text{r}}'\sim \boldsymbol{Z}_{R}}\left[k(\boldsymbol{z},\boldsymbol{z}_{\text{r}}')\right] \\
    &-2\mathbb{E}_{z_{\text{e}}\sim\boldsymbol{Z}_{E},z_{\text{r}}\sim \boldsymbol{Z}_{R}}\left[k(\boldsymbol{z}_{\text{e}},\boldsymbol{z}_{\text{r}})\right].
\end{aligned}
\end{equation}

\subsection{Edit Trajectory Invariant Learning}\label{sec:invariant}
With the editing OOD formulation in Section \ref{sec:setting} and the overall risk composed of supervised signals on two distributional shifts in Section \ref{sec:alignment}, we now introduce an invariant learning paradigm to discern and stabilize the edit trajectories across diverse cross-modal environments.
Our goal is to optimize edited model parameters $\phi_e$ to minimize the risk across environments, which requires exploiting invariance and specificity in causal pathways activated by edit.

\nosection{Transformation into IRM Problem}
To extract invariant trajectories, we invoke the IRM principle and employ a classifier $\omega$ which maps environment features to predictions \citep{lai2024invariant}.

\begin{proposition}[Equivalence between OOD-$\omega$ and IRM]
\label{prop:equivalence}
Under the condition that the environment variability is channeled through the classifier $\omega$, it satisfies the identity $\mathcal{R}_{\text{edit}}(\phi_e, e) \equiv \mathcal{R}_{\text{edit}}(\omega(e) \circ \phi_e)$. The OOD editing objective in Eq.(\ref{eq:ood}) admits the following equivalent IRM formulation:
\[
\min_{\phi_e} \max_{\omega \in \varSigma} \mathbb{E}_{(\mathbf{m}_{\text{e}}, \mathbf{x}_{\text{e}},\mathbf{y}_{\text{e}})\sim \mathcal{P}_e(\mathbf{M},\mathbf{X},\mathbf{Y})} \mathcal{R}_{\text{edit}}(\omega(\mathbf{m}_{\text{e}}, \mathbf{x}_{\text{e}},\mathbf{y}_{\text{e}}) \circ \phi_e).
\]
\end{proposition}
\begin{proof}
The proof can be found in Appendix \ref{app:ood-w}.
\end{proof}


\nosection{Invariant Learning in Editing Trajectory}
Directly optimizing the OOD-$\omega$ objective is intractable due to the need to evaluate the supremum over $\mathcal{E}$.
To this end, we reformulate it within a measure-theoretic framework inspired by the connection between IRM and Total Variations (TV) \citep{chan2006total}.
%
%
The TV operator typically employed to measure the global variability bound of a function.
For a function $f$ defined on a measure space $(\Omega,\mathcal{F}_{\Omega},\nu)$, the TV seminorm is given by
\begin{equation}
\scriptsize
    TV(f) := \sup\left\{ \int_{\Omega} f(\nu) \, \mathrm{div}\, g(\nu)  d\nu : g \in C_c^1(\Omega, \mathbb{R}^d), \|g\|_{\infty} \leq 1 \right\},
\end{equation}
where $g$ is a differentiable vector function supported compactly in $\Omega$ and $\mathrm{div} g$ denotes its divergence.
Based on the Coarea Formula \citep{chan2006total}, the canonical TV-$\ell_1$ can be derived to recover a clean signal $f$ from a noisy observation $\tilde{f}$ by solving the variational problem:
\begin{equation}
\small
    \inf_{f \in L^2(\Omega)} \left\{ \int_{\Omega} |\nabla f| + \lambda \int_{\Omega} (f - \tilde{f})^2  d \nu \right\}
\end{equation}
%
Here, TV-$\ell_1$ model pres sharp discontinuities while effectively removing noise and fine-scale details.
Correspondingly, we treat the environment-induced variations in the risk function $\mathcal{R}_{\text{edit}}(\omega \circ \phi_e)$ as noise perturbing the ideal and invariant edit trajectory.
The goal of editing is to \textit{denoise} the risk, recovering a piecewise-constant profile that is robust to spurious cross-modal prompting changes.
Inspired by \cite{lai2024invariant}, we further absorb TV-$\ell_1$ penalty into our editing IRM objective as
\begin{equation}
\small
\begin{aligned}
    \min_{\phi_e} \{ \mathbb{E}_{\omega}&[\mathcal{R}_{\text{rel}}(\omega \circ \phi_e)+\mathcal{R}_{\text{loc}}(\omega \circ \phi_e)+\mathcal{R}_{\text{gen}}(\omega \circ \phi_e)]\\ 
    &+ \lambda_{\phi_e} \left( \mathbb{E}_{\omega}[|\nabla_{\omega} \mathcal{R}_{\text{edit}}(w \circ \phi_e)|] \right)^2 \}.
    \label{eq:tv-ood}
\end{aligned}
\end{equation}
The first term represents the basic risk of editing, while the second term promotes invariance by encouraging the generalization risk to be insensitive to environmental changes.
This form directly addresses the dual requirements of precise \textit{knowledge assimilation} and \textit{controlled generalization}.

\begin{proposition}[IRM-TV objective Achieves Editing OOD with a varying $\lambda$]
\label{prop:ood-tv}
The balancing parameter $\lambda$ should vary with editing parameters $\phi_e$ to achieve editing OOD.
For each $\phi_e$, if $\mathbb{E}_{\omega}[|\nabla_{\omega} \mathcal{R}_{\text{edit}}(\omega \circ \phi_e)|] > 0$, there exists a non-negative $\lambda_{\phi_e}$, such that
\begin{equation}
\small
\begin{aligned}
    \max_{e \in \mathcal{E}} &\mathcal{R}_{\text{edit}}(\phi_e, e) = \mathbb{E}_{\omega}[\mathcal{R}_{\text{rel}}(\omega \circ \phi_e)+\mathcal{R}_{\text{loc}}(\omega \circ \phi_e)+\\
    &\mathcal{R}_{\text{gen}}(\omega \circ \phi_e)] + \lambda_{\phi_e} \left( \mathbb{E}_{\omega}[|\nabla_{\omega} \mathcal{R}_{\text{edit}}(\omega \circ \phi_e)|] \right)^2.
\end{aligned}
\end{equation}
Besides, the optimality of $\phi_e$ for IRM-TV form is equivalent to its optimality for OOD-$\omega$.
\end{proposition}
\begin{proof}
The proof can be found in Appendix \ref{app:irm-tv-ood}.
\end{proof}


\nosection{Optimization on Editing IRM-TV}
To solve Eq.(\ref{eq:tv-ood}), we treat $\lambda_{\phi_e}$ as a Lagrangian multiplier and parameterize it as a function $\lambda(\pi, \phi_e)$ of both the editing model parameters $\phi_e$ and an auxiliary dual parameter set $\delta$.
We derive the Lagrangian function for the editing IRM-TV objective as 
\begin{equation}
\small
\begin{aligned}
    \mathcal{G}(\delta, \phi_e) = &\mathbb{E}_{\omega}[\mathcal{R}_{\text{rel}}(\omega \circ \phi_e)+\mathcal{R}_{\text{loc}}(\omega \circ \phi_e)+\mathcal{R}_{\text{gen}}(\omega \circ \phi_e)]] \\
    &+ \lambda(\delta, \phi_e) \left( \mathbb{E}_{\omega}[|\nabla_{\omega} \mathcal{R}_{\text{edit}}(\omega \circ \phi_e)|] \right)^2.
    \label{eq:13}
\end{aligned}
\end{equation}
Denote the risk sum as $\mathcal{R}_{\text{edit}}$, we derive it into a primal-dual optimization as \citep{wang2025out}
\begin{equation}
\small
\begin{aligned}
    &\min _{\phi_e} \max_{\delta} \mathcal{G}(\delta, \phi_e):=\min _{\phi_e}\{\mathbb{E}_\omega[\mathcal{R}_{\text{edit}}(w \circ \phi_e)]\\
    &+\max _{\delta}\left[\lambda(\delta, \phi_e)\left(\mathbb{E}_\omega\left[\left|\nabla_\omega \mathcal{R}_{\text{edit}}(\omega \circ \phi_e)\right|\right]\right)^2\right]\},
    \label{eq:14}
\end{aligned}
\end{equation}
where the \textit{primal variable} $\phi_e$ is optimized to minimize the overall risk, and \textit{dual variable} $\delta$ is optimized to maximize the TV penalty.
To solve it, an adversarial learning procedure is adopted, alternating between updating $\phi_e$ and $\delta$ with adaptive learning rates $\gamma_1$ and $\gamma_2$:
\begin{equation}
\small
\begin{aligned}
\phi_e^{(k+1)} &= \phi_e^{(k)} - \gamma_1^{(k)} \cdot \partial_{\phi_e} \mathcal{G}(\delta^{(k)}, \phi_e^{(k)}), \\
\delta^{(k+1)} &= \delta^{(k)} + \gamma_2^{(k)} \cdot \nabla_{\delta} \mathcal{G}(\delta^{(k)}, \phi_e^{(k+1)}).
\label{eq:15}
\end{aligned}
\end{equation}
The computation process of gradient $\nabla_{\delta} \mathcal{G}$ and subgradient $\partial_{\phi_e} \mathcal{G}$ are presented in Appendix \ref{app:gradients}.
Consequently, after optimizing two variables, we obtain an edit model $\phi_e$ that is both accurate and contained while being robustly generalizable through invariant mechanisms.

\section{Experiment}
\subsection{Experimental Setup}

\begin{table*}[t]
\centering
\scriptsize
\caption{Editing performance (\%). \textit{Rel.}, \textit{Gen.}, \textit{T-Loc.}, \textit{M-Loc.}, denote Reliability, Generality, Text Locality, and Image Locality. The higher scores are highlighted in bold. All improvements are significant with $p$-value $<$ 0.05 based on $t$-tests.}
\label{tab:compare}
\resizebox{\linewidth}{!}{
\begin{tabular}{ccc|cccc|cccc}
\toprule
\multirow{2.2}{*}{\textbf{Model}} &\multirow{2.2}{*}{\textbf{Type}} &\multirow{2.2}{*}{\textbf{Method}} 
& \multicolumn{4}{c|}{\textbf{Editing VQA (E-VQA)}}
& \multicolumn{4}{c}{\textbf{Editing Image Caption (E-IC)}}\\
\cmidrule(lr){4-7} \cmidrule(lr){8-11} 
& & & Rel.$\uparrow$ & Gen.$\uparrow$ & T-Loc.$\uparrow$ & M-Loc.$\uparrow$ & Rel.$\uparrow$ & Gen.$\uparrow$ & T-Loc.$\uparrow$ & M-Loc.$\uparrow$\\

\midrule

\multirow{12}{*}{\rotatebox{90}{\textbf{BLIP2-OPT 2.7B}}} & / & Pre-edited & 25.85	&26.37	&99.38	&92.83	&0	&0	&99.79	&94.93\\
\cmidrule(lr){2-11}
&Naive & FT & 100 & 100 & 93.94 & 64.79 & 100 & 0 & 78.79 & 29.58\\
& Model-extend& IKE  & 99.71 & 99.62 & 47.74 & 2.53 & 94.40 & 88.00 & 50.43 & 2.87\\ 
& Model-extend& SERAC & 97.60 & 97.30 & 100 & 3.21 & 99.71 & 99.71 & 100 & 2.64\\ 
\cmidrule(lr){2-11}
& Model-extend& WISE  & 100 & 83.33 & 40.94 & 16.89 & 100 & 85.93 & 33.61 & 11.89\\ 
& Model-extend& WISE+\modelname  & \textbf{100} & \textbf{83.33} & \textbf{41.24} & 13.83 & \textbf{100} & \textbf{87.33} & \textbf{39.37} & \textbf{14.77}\\ 
\cmidrule(lr){2-11}
& Para-adjust&  MEND & 97.80 & 97.20 & 99.68 & 94.23 & 77.90 & 62.80 & 98.14 & 78.86\\ 
& Para-adjust& MEND+\modelname & 97.60 & \textbf{97.20} & 99.52 & 91.75 & \textbf{79.40} & \textbf{64.40} & \textbf{99.01} & \textbf{86.14}\\ 
\cmidrule(lr){2-11}
& Model-extend& T-Patcher  & 80.35 & 77.82 & 87.14 & 85.28 & 72.78 & 72.75 & 71.59 & 80.49\\ 
& Model-extend& T-Patcher+\modelname  & \textbf{81.85} & \textbf{80.47} & 86.25 & \textbf{85.37} & \textbf{73.44} & \textbf{74.28 }& 71.18 & \textbf{81.67}\\ 
\cmidrule(lr){2-11}
&Both & UniKE  & 94.32 & 87.18 & 95.98 & 93.15 & 74.01 & 73.84 & 76.09 & 82.36\\ 
&Both & UniKE+\modelname  & \textbf{96.58} & \textbf{89.34} & \textbf{96.17} & \textbf{93.27} & \textbf{74.52} & \textbf{75.49} & \textbf{76.65} & \textbf{83.28}\\

\bottomrule
\toprule

\multirow{12}{*}{\rotatebox{90}{\textbf{MiniGPT-4 7B}}} & / & Pre-edited & 19.21	&24.08	&99.44	&91.56	&0	&0	&99.79	&94.93\\
\cmidrule(lr){2-11}
&Naive & FT & 100 & 100 & 97.50 & 40.85 & 100 & 0 & 95.00 & 39.83\\
& Model-extend& IKE  & 99.95 & 99.90 & 50.02 & 3.31 & 90.30 & 90.00 & 51.49 & 4.27\\ 
& Model-extend& SERAC & 91.70 & 98.60 & 99.99 & 3.72 & 83.60 & 93.10 & 99.99 & 4.65\\ 
\cmidrule(lr){2-11}
& Model-extend& WISE  & 100 & 100 & 90.10 & 52.15 & 100 & 91.58 & 92.81 & 70.68\\ 
& Model-extend& WISE+\modelname  & \textbf{100} & 97.50 & \textbf{92.39} & \textbf{62.14} & \textbf{100} & 90.04 & \textbf{94.54} & \textbf{73.17}\\ 
\cmidrule(lr){2-11}
& Para-adjust&  MEND & 96.20 & 96.00 & 99.42 & 88.25 & 77.80 & 74.60 & 99.28 & 87.85\\ 
& Para-adjust& MEND+\modelname  & \textbf{97.00} & \textbf{97.00} & \textbf{99.52} & \textbf{88.61} & \textbf{78.60 }& 74.20 & \textbf{99.36} & 86.77\\ 
\cmidrule(lr){2-11}
& Model-extend& T-Patcher  & 70.56 & 68.79 & 64.45 & 81.77 & 69.54 & 68.95 & 63.59 & 81.34\\ 
& Model-extend& T-Patcher+\modelname  & \textbf{72.38} & \textbf{72.11} & \textbf{65.29} & \textbf{82.93} & \textbf{71.42} & \textbf{70.98} & \textbf{65.03} & \textbf{82.75}\\ 
\cmidrule(lr){2-11}
&Both & UniKE  & 84.32 & 81.29 & 78.45 & 85.81 & 72.18 & 70.41 & 68.53 & 84.59\\ 
&Both & UniKE+\modelname  & \textbf{85.14} & \textbf{83.23} & \textbf{79.35} & \textbf{86.56} & \textbf{73.06} & \textbf{71.58} & \textbf{69.46} & \textbf{85.12}\\

\bottomrule

\end{tabular}
}
\vspace{-0.20 in}
\end{table*}

\nosection{Datasets \& Backbones \& Evaluation}
In line with previous work \citep{pan2024towards}, we conduct experiments on the MMEdit benchmark \citep{cheng2023can}, encompassing two sub-tasks, \ie Editing VQA (E-VQA) and Editing Image Captioning (E-IC).
Under this benchmark, we choose BLIP2-OPT \citep{li2023blip} and MiniGPT-4 \citep{zhu2023minigpt} as the base MLLM.
We utilize Reliability, Generality, T-Locality, and M-Locality as the evaluation metrics.

\nosection{Baseline Methods}
We compare \modelname~with four types of baselines:
(1) \textit{Fine-tuning}: FT directly tunes the last three layers of MLLM.
(2) \textit{Parameter-adjusting unimodal editing}: MEND \citep{mitchell2021fast}.
(3) \textit{Model-extending unimodal editing}: IKE \citep{zheng2023can}, SERAC \cite{mitchell2022memory}, T-Patcher \citep{huang2023transformer}, WISE \citep{wang2024wise}.
(4) \textit{Integrate parameter-adjusting and model-extending multimodal editing}: UniKE \citep{pan2024towards}.
\modelname~serves as a plug-and-play universal framework, capable of being seamlessly integrated into any editing model that relies on loss-based optimization.
%

\nosection{Implementation Details}
For the generality risk in \modelname, we employ a Gaussian RBF kernel with a multi-scale bandwidth strategy as the kernel function for MMD.
To achieve adaptive TV-$\ell_1$ penalty, we utilize a three-layer MLP with ReLU activations for the IRM-TV optimization, Xavier initialization for weights, and Softplus activation at the output to ensure positivity.
We choose Adam as the optimizer, and vary the learning rates in $\{0.0001,0.001,0.005,0.01\}$ for the IRM-TV network.
For all experiments, we repeat them five times and report the mean value of the results.
We conduct all of our experiments on an Ubuntu OS that contains 8 NVIDIA A40 GPUs.
%

\begin{table*}[t]
\vspace{-0.05in}
\centering
\small
\caption{\revise{Results of long-term editing on BLIP2-OPT. Here we set the sequential edit step as T=5 and T=10, respectively.}}
\vspace{-0.05in}
\label{tab:long-term}
\resizebox{\linewidth}{!}{
\begin{tabular}{cc|cccc|cccc}
\toprule
\multirow{2.2}{*}{\revise{\textbf{Dataset}}} 
&\multirow{2.2}{*}{\revise{\textbf{Model}}} 
& \multicolumn{4}{c|}{\revise{\textbf{T=5}}}
& \multicolumn{4}{c}{\revise{\textbf{T=10}}}\\
\cmidrule(lr){3-6} \cmidrule(lr){7-10} 
& & \revise{Rel.$\uparrow$} & \revise{Gen.$\uparrow$} & \revise{T-Loc.$\uparrow$} & \revise{M-Loc.$\uparrow$} & \revise{Rel.$\uparrow$} & \revise{Gen.$\uparrow$} & \revise{T-Loc.$\uparrow$} & \revise{M-Loc.$\uparrow$}\\
\midrule

\multirow{4}{*}{\textbf{Editing VQA (E-VQA)}}
& \revise{WISE} & \revise{44.50}	& \revise{34.75}	& \revise{0.40}	& \revise{0.15} & \revise{28.50} & \revise{24.55} & \revise{0.63} & \revise{0.15}\\
& \revise{WISE+\modelname} & \revise{49.42}	& \revise{43.52}	& \revise{0.80}	& \revise{0.15} & \revise{43.33} & \revise{24.22} & \revise{0.81} & \revise{0.15}\\
& \revise{UniKE} & \revise{90.28} & \revise{80.26} & \revise{91.41} & \revise{89.37} & \revise{86.52} & \revise{76.58} & \revise{87.64} & \revise{86.31}\\
& \revise{UniKE+\modelname} & \revise{92.63} & \revise{83.59} & \revise{92.38} & \revise{89.95} & \revise{89.79} & \revise{81.25} & \revise{89.35} & \revise{87.54}\\

\toprule

\multirow{4}{*}{\textbf{Editing Image Caption (E-IC)}}
& \revise{WISE} & \revise{84.31} & \revise{65.49} & \revise{0.76} & \revise{0.14} & \revise{75.96} & \revise{55.56} & \revise{0.71} & \revise{0.14}\\
& \revise{WISE+\modelname} & \revise{86.53} & \revise{66.94} & \revise{0.94} & \revise{0.14} & \revise{84.64} & \revise{61.70} & \revise{0.77} & \revise{0.14}\\
& \revise{UniKE} & \revise{70.16} & \revise{71.45} & \revise{72.09} & \revise{79.52} & \revise{63.54} & \revise{64.71} & \revise{66.29} & \revise{73.25}\\
& \revise{UniKE+\modelname} & \revise{71.05} & \revise{73.22} & \revise{72.68} & \revise{80.77} & \revise{65.87} & \revise{68.82} & \revise{67.11} & \revise{76.59}\\

\bottomrule

\end{tabular}
}
\vspace{-0.2 in}
\end{table*}

\subsection{Performance on One-step Knowledge Editing}
To evaluate editing performance, we conduct one-step editing experiments.
From Table \ref{tab:compare}, we can find:
1) \textbf{Previous methods fail to achieve balanced performance across all metrics when applied to multimodal editing tasks.}
Model-extending methods frequently suffer from poor locality, while parameter-adjusting methods often exhibit limited generality.
For instance, SERAC achieves high reliability and generality on E-VQA, but its M-Locality drops drastically.
MEND shows a significant generality gap with other SOTAs.
2) \textbf{\modelname~demonstrates strong adaptability across diverse baselines and consistently improves four evaluation metrics.}
On E-VQA with MiniGPT-4, T-Patcher+\modelname~improves generality with the promotion ratio as 4.82\%.
WISE+\modelname~improves M-Locality by 19.2\% with MiniGPT-4 on E-VQA, while T-Locality by 17.2\% with BLIP2 on E-IC.
UniKE+\modelname~outperforms UniKE on all metrics.
%
\revise{3) \textbf{The balanced improvement across metrics underscores that effective OOD generalization equates to holistic performance elevation, not merely gains in the generality dimension.} \modelname~accurately determines the generalization boundary, and its core contribution is extracting invariant editing trajectories to both mitigate causal underfit and causal overfit, thus resolving the trade-off between locality and generality.}

\subsection{\revise{Performance on Long-term Knowledge Editing}}
\revise{
%
Following \cite{pan2024towards}, we typically set the $T$-step sequential editing scenario, where the model is edited sequentially for each instance in the editing set with a capacity of $T$.
After the $T$-th edit, we evaluate the post-edit MLLM.
We report the results for $T=5$ and $T=10$ on both E-VQA and E-IC tasks.
From Table \ref{tab:long-term}, we find:
1) Unimodal editors like WISE fail catastrophically in multimodal long-term editing, particularly in preserving locality and generality.
On E-VQA, WISE's T-Loc. collapses to near zero, demonstrating the rigid editing mapping cannot adaptively modify MLLM's causal reasoning.
2) Even specialized multimodal editors like UniKE exhibit performance decay over time.
This indicates that without explicit invariance learning, sequential edits cause interference and erode previously learned knowledge.
3) \modelname~consistently mitigates this decay and enhances stability. 
By learning invariant trajectories, \modelname~preserves higher reliability, generality, and locality, and the improvement becomes more pronounced as $T$ increases.
These results prove the ability of \modelname~to discern and stabilize core causal features against the variations introduced by successive edits.
}

\begin{table}[t]
\centering
\caption{Ablation studies on each OOD risk.}
\vspace{-0.05in}
\label{tab:ab-risk}
\resizebox{1.0\linewidth}{!}{
\begin{tabular}{c|cccc|cccc}
\toprule
\multirow{2.2}{*}{\textbf{Invariants}} 
& \multicolumn{4}{c|}{\textbf{Editing VQA (E-VQA)}}
& \multicolumn{4}{c}{\textbf{Editing Image Caption (E-IC)}}\\
\cmidrule(lr){2-5} \cmidrule(lr){6-9} 
& Rel.$\uparrow$ & Gen.$\uparrow$ & T-Loc.$\uparrow$ & M-Loc.$\uparrow$ & Rel.$\uparrow$ & Gen.$\uparrow$ & T-Loc.$\uparrow$ & M-Loc.$\uparrow$\\
\midrule
w/o $\mathcal{R}_{rel}$ &0	&0	&99.85	&98.63	&0	&0	&95.81	&97.22\\
w/o $\mathcal{R}_{loc}$ &96.65	&89.51	&74.36	&71.25	&74.62	&75.45	&64.49	&73.19\\
w/o $\mathcal{R}_{gen}$ &96.49	&86.59	&95.97	&93.54	&74.60	&71.24	&75.83	&83.60\\
\modelname &96.58	&89.34	&95.46	&93.27	&74.52	&75.49	&75.65	&83.28\\

\bottomrule

\end{tabular}
}
\vspace{-0.1 in}
\end{table}

\subsection{In-Depth Analysis}

\nosection{Ablations of OOD Risks}
We conduct ablation studies by removing each risk separately.
The results in Table \ref{tab:ab-risk} show:
1) \textbf{Reliability risk is essential for knowledge assimilation.}
The removal of $\mathcal{R}_{rel}$ causes Rel. and Gen. to collapse, showing MLLM fails to learn the intended knowledge.
2) \textbf{Locality risk constraints editing within scope.}
Ablating $\mathcal{R}_{loc}$ leads to a decrease in T-Loc. and M-Loc., indicating edit effects spill over into irrelevant knowledge areas, causing causal-overfit and violating locality. 
\revise{3) \textbf{Generality risk facilitates semantic generalization.}}
\revise{The addition of $\mathcal{R}{gen}$ yields a substantial gain in the Gen. metric (e.g., from 86.59 to 89.34 on E-VQA and from 71.24 to 75.49 on E-IC). While it induces minimal fluctuations in Rel. and minor variations in Loc., this aligns with our OOD formulation where the three risks exhibit inherent cross-effects. $\mathcal{R}{gen}$ primarily works by aligning semantic-neighboring samples with the edited instance in the latent space, which successfully promotes invariant feature learning to prevent causal underfit. The slight impact on locality can be attributed to the potential reinforcement of local concept-output associations as a byproduct of this semantic alignment process.}


\begin{table}[t]
\centering
\small
\caption{Ablation studies on the MMD alignment.}
\vspace{-0.05in}
\label{ablation-mmd}
\begin{tabular}{l|cccc}
\toprule
Invariant & Rel. & Gen. & T-Loc. & M-Loc. \\
\midrule
MMD-s RBF & 79.40 & 64.40 & 99.01 & 86.14 \\
MMD-s Linear & 78.80 & 63.20 & 98.49 & 80.75 \\
MMD-m RBF & 76.81 & 63.59 & 99.00 & 85.73 \\
Contrast & 76.40 & 63.80 & 99.00 & 89.91 \\
\bottomrule
\end{tabular}
\vspace{-0.1in}
\end{table}

\nosection{Effects of Maximum Mean Discrepancy Alignment}
We perform ablations with invariants:
(a) \textit{MMD-s RBF} denotes MMD with a Radial Basis Function (RBF) kernel and a single rephrase prompt.
(b) \textit{MMD-s Linear} with linear kernel.
(c) \textit{MMD-m RBF} utilizes multiple rephrase prompts.
(d) \textit{Contrast} replace MMD with contrastive learning.
Results in Table \ref{ablation-mmd} show that \textit{MMD-s RBF} achieves the most balanced and effective performance.
\textit{MMD-s Linear} is less effective at capturing cross-modal semantic distributions.
The gap between \textit{Contrast} and \textit{MMD-s RBF} underscores advantages of a distribution-level alignment objective over instance-level.
An insightful finding is that using multiple rephrase prompts yields no additional benefit.
The potential reason is that a single rephrase prompt provides a focused semantic transformation path, while multiple prompts introduce noisy variations which might lead to spurious correlations.
\begin{table}
\centering
\small
\caption{\revise{Results on other MLLMs.}}
\vspace{-0.05in}
\label{table:llava}
\begin{tabular}{l|cccc}
\toprule
\revise{E-VQA} & \revise{Rel.} & \revise{Gen.} & \revise{T-Loc.} & \revise{M-Loc.} \\
\midrule
\revise{WISE} & \revise{100} & \revise{71.42} & \revise{91.51} & \revise{93.75} \\
\revise{WISE+\modelname} & \revise{100} & \revise{72.01} & \revise{94.40} & \revise{95.47} \\
\midrule
\revise{E-IC} & \revise{Rel.} & \revise{Gen.} & \revise{T-Loc.} & \revise{M-Loc.} \\
\midrule
\revise{WISE} & \revise{99.89} & \revise{81.28} & \revise{92.69} & \revise{94.63} \\
\revise{WISE+\modelname} & \revise{99.98} & \revise{82.22} & \revise{92.74} & \revise{95.92} \\
\bottomrule
\end{tabular}
\vspace{-0.1in}
\end{table}

\begin{figure}
\centering
\includegraphics[width=0.8\linewidth]{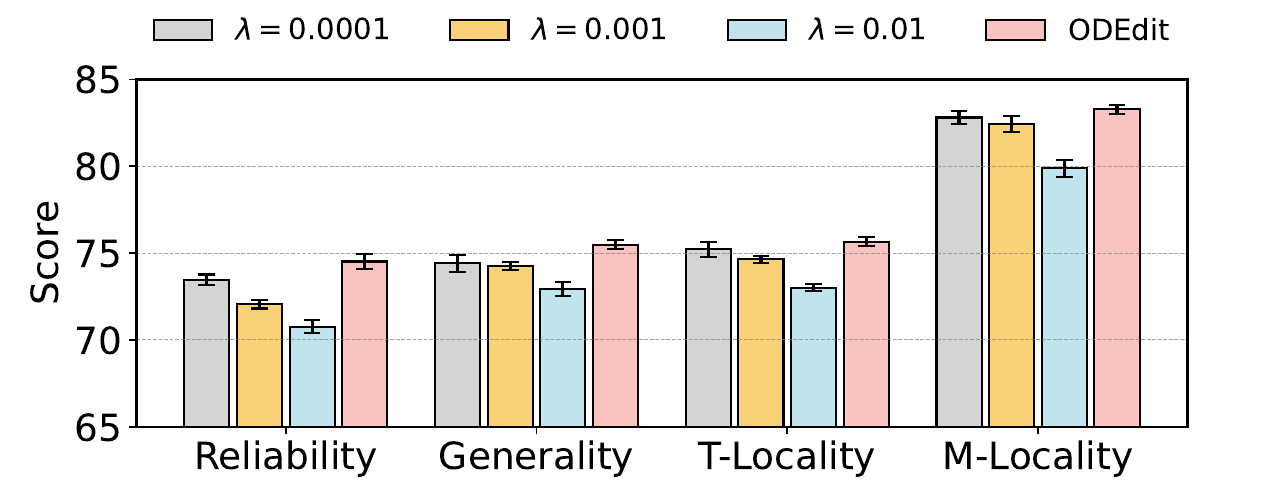}
\caption{Ablation studies of IRM-TV optimization.}
\label{fig:ablation_ood}
\vspace{-0.2in}
\end{figure}

\nosection{Performance on other MLLM backbones.}
\revise{
We further conduct editing on other MLLMs, \ie LLaVA \citep{liu2023visual}.
From Table \ref{table:llava}, \modelname~enhances WISE across all metrics on LLaVA, with particularly notable gains in generality and locality. These robust improvements on a distinct MLLM architecture underscore the strong generalizability of \modelname, which stems from its core design of learning invariant editing trajectories that effectively suppress spurious correlations across diverse model backbones.
}

\begin{table*}
\centering
\caption{\revise{Computational cost comparison on E-IC. Memo. = Memory usage, Edit-T/sp = Editing time per step, Train-T/sp = Training time per step, All-T/sp = Total time per step.}}
\label{tab:compute-cost}
\resizebox{1.0\linewidth}{!}{
\begin{tabular}{c|ccccc|ccccc}
\toprule
\multirow{1.8}{*}{\revise{\textbf{Models}}} 
& \multicolumn{5}{c|}{\revise{\textbf{BLIP2-OPT}}}
& \multicolumn{5}{c}{\revise{\textbf{MiniGPT-4}}}\\
\cmidrule(lr){2-6} \cmidrule(lr){7-11} 
& \revise{Memo.} & \revise{Edit-T/sp} & \revise{Train-T/sp} & \revise{All-T/sp} & \revise{Steps.} 
& \revise{Memo.} & \revise{Edit-T/sp} & \revise{Train-T/sp} & \revise{All-T/sp} & \revise{Steps.}\\
\midrule

\revise{WISE} & \revise{28.47GB} & \revise{0.401} & \revise{0.023} & \revise{0.424} & \revise{7} 
              & \revise{36.59GB} & \revise{0.473} & \revise{0.035} & \revise{0.508} & \revise{8}\\

\revise{WISE+\modelname} & \revise{47.53GB} & \revise{0.442} & \revise{0.043} & \revise{0.485} & \revise{7} 
                        & \revise{69.36GB} & \revise{0.533} & \revise{0.072} & \revise{0.605} & \revise{8}\\

\revise{MEND} & \revise{14.75GB} & \revise{1.369} & \revise{0.018} & \revise{1.387} & \revise{25000} 
              & \revise{25.30GB} & \revise{1.676} & \revise{0.188} & \revise{1.865} & \revise{10000}\\

\revise{MEND+\modelname} & \revise{36.05GB} & \revise{1.480} & \revise{0.116} & \revise{1.596} & \revise{45000} 
                        & \revise{62.80GB} & \revise{1.861} & \revise{0.204} & \revise{2.066} & \revise{15000}\\

\bottomrule
\end{tabular}
}
\end{table*}

\begin{figure}
\centering

\begin{subfigure}{0.32\linewidth}
    \centering
    \includegraphics[width=\linewidth]{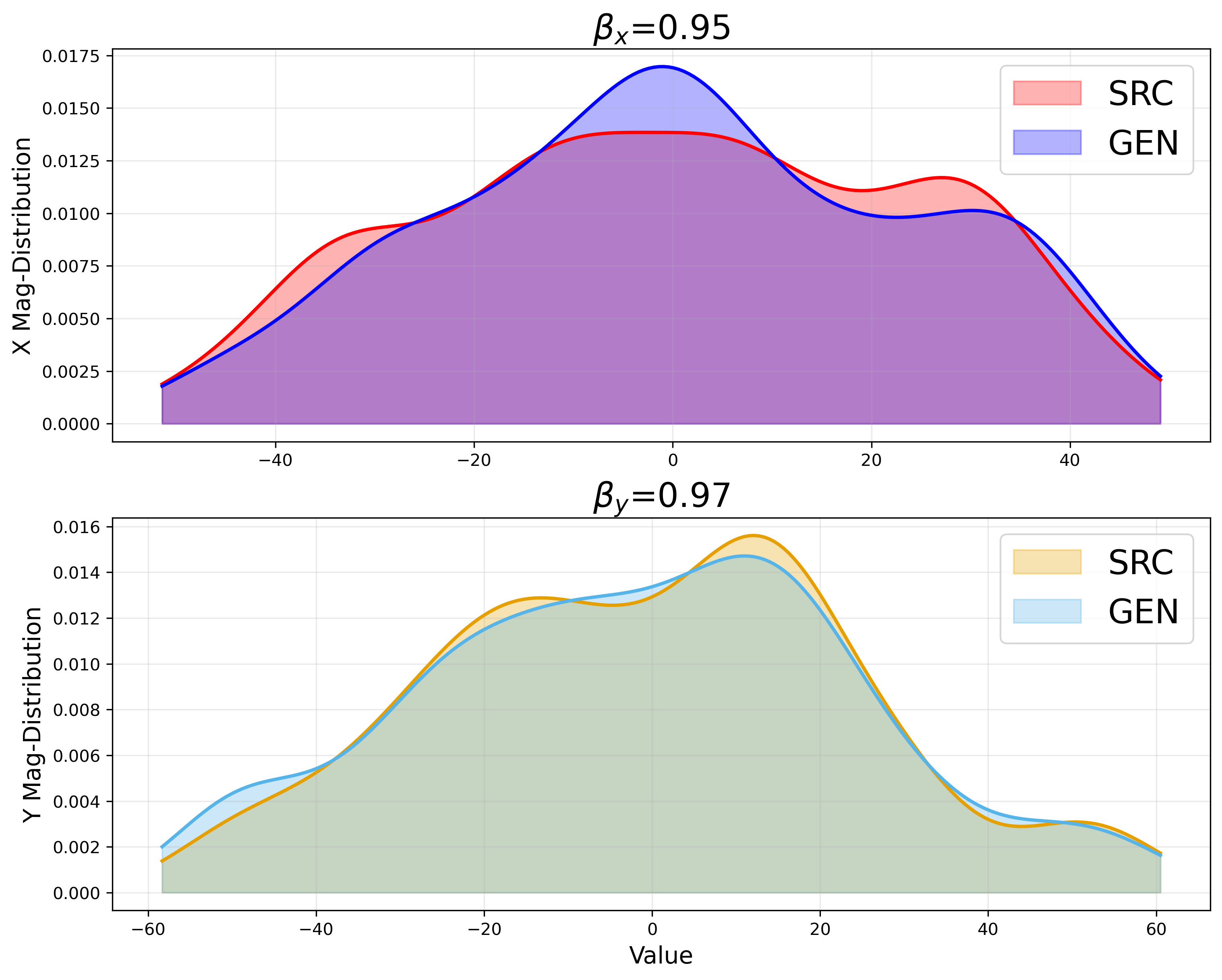}
    \caption{\revise{before-edit}}
    \label{fig:subfig:t1}
\end{subfigure}
\hfill
\begin{subfigure}{0.32\linewidth}
    \centering
    \includegraphics[width=\linewidth]{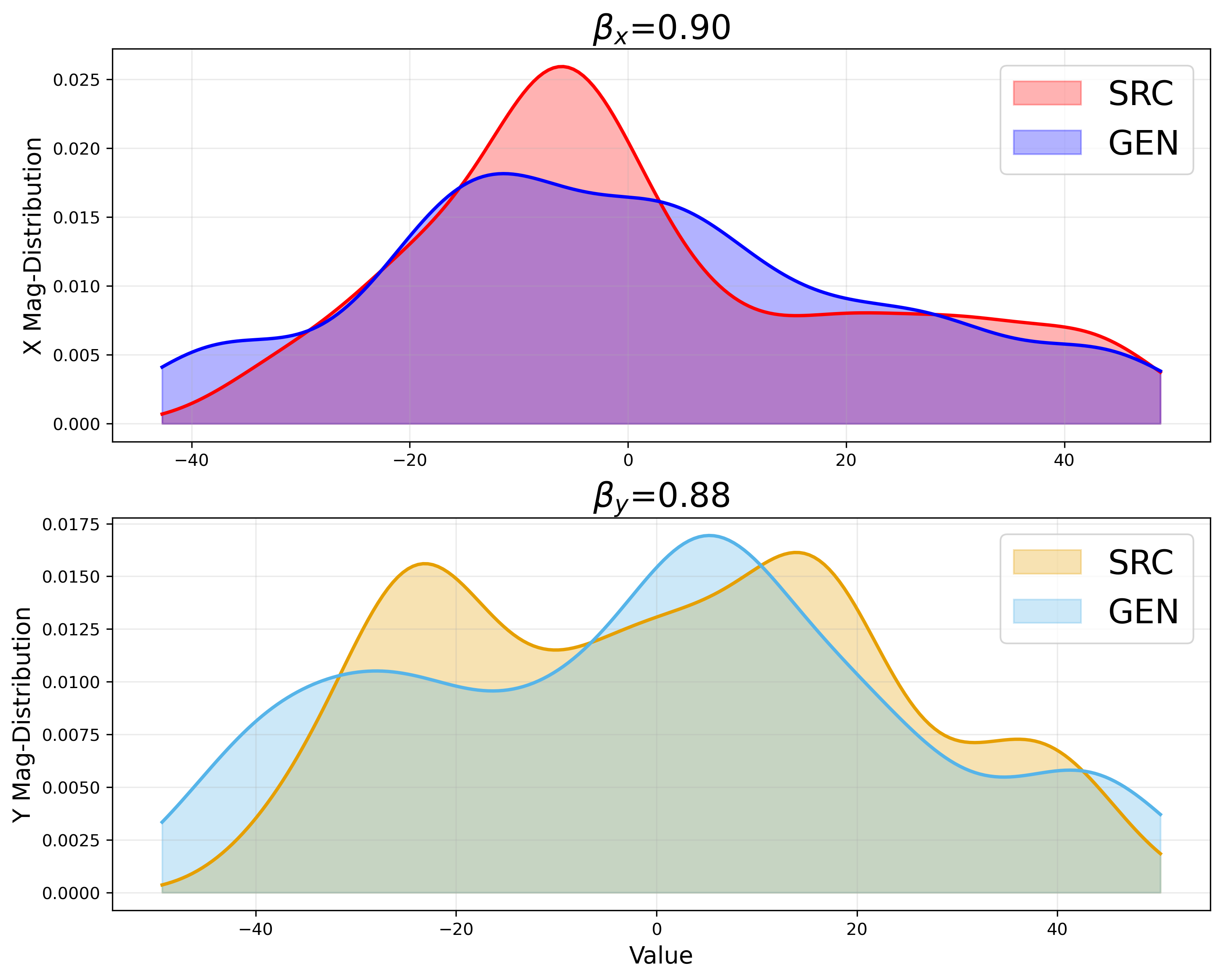}
    \caption{\revise{in editing}}
    \label{fig:subfig:t2}
\end{subfigure}
\hfill
\begin{subfigure}{0.32\linewidth}
    \centering
    \includegraphics[width=\linewidth]{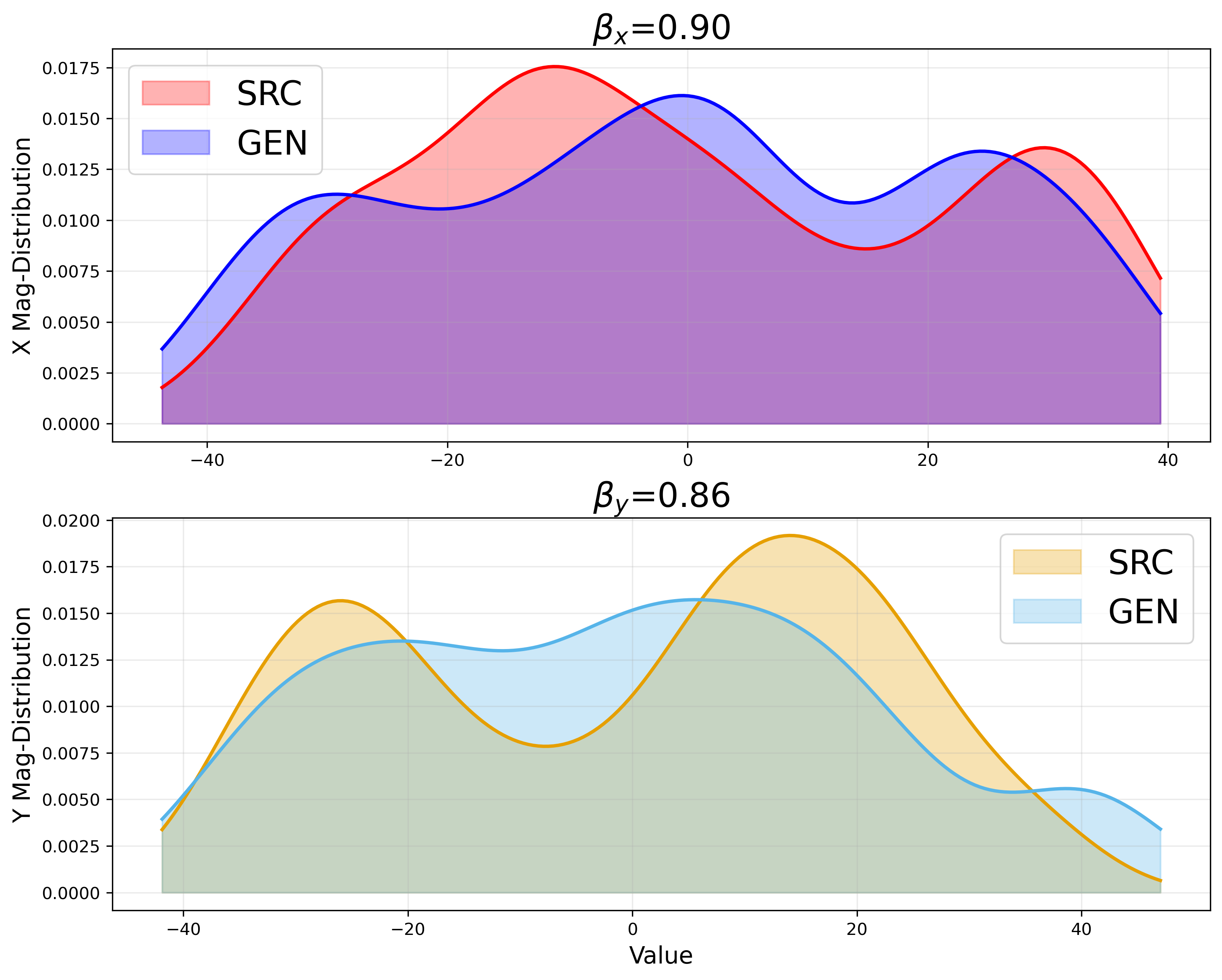}
    \caption{\revise{converged}}
    \label{fig:subfig:t3}
\end{subfigure}

\begin{subfigure}{0.32\linewidth}
    \centering
    \includegraphics[width=\linewidth]{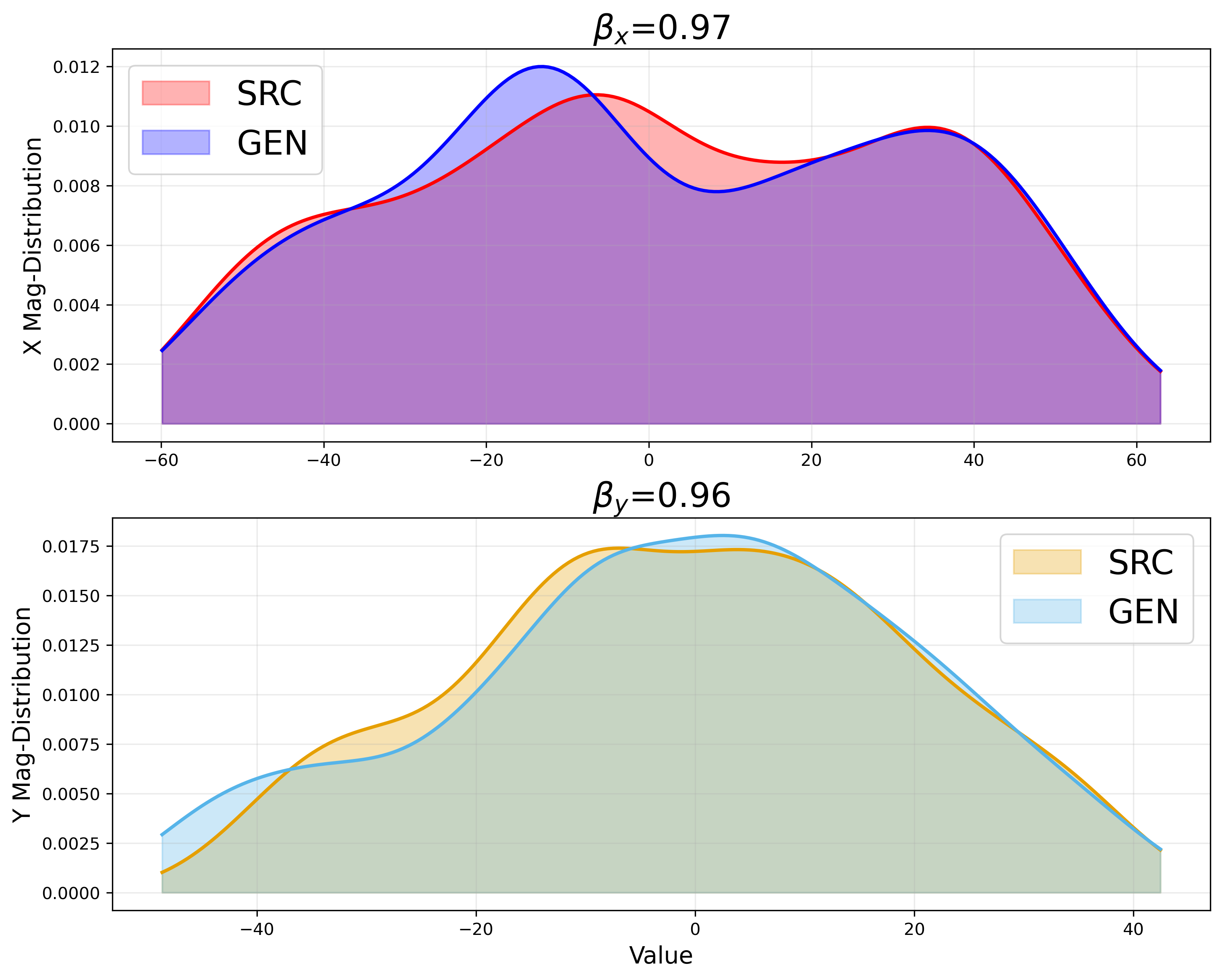}
    \caption{\revise{before-edit}}
    \label{fig:subfig:t4}
\end{subfigure}
\hfill
\begin{subfigure}{0.32\linewidth}
    \centering
    \includegraphics[width=\linewidth]{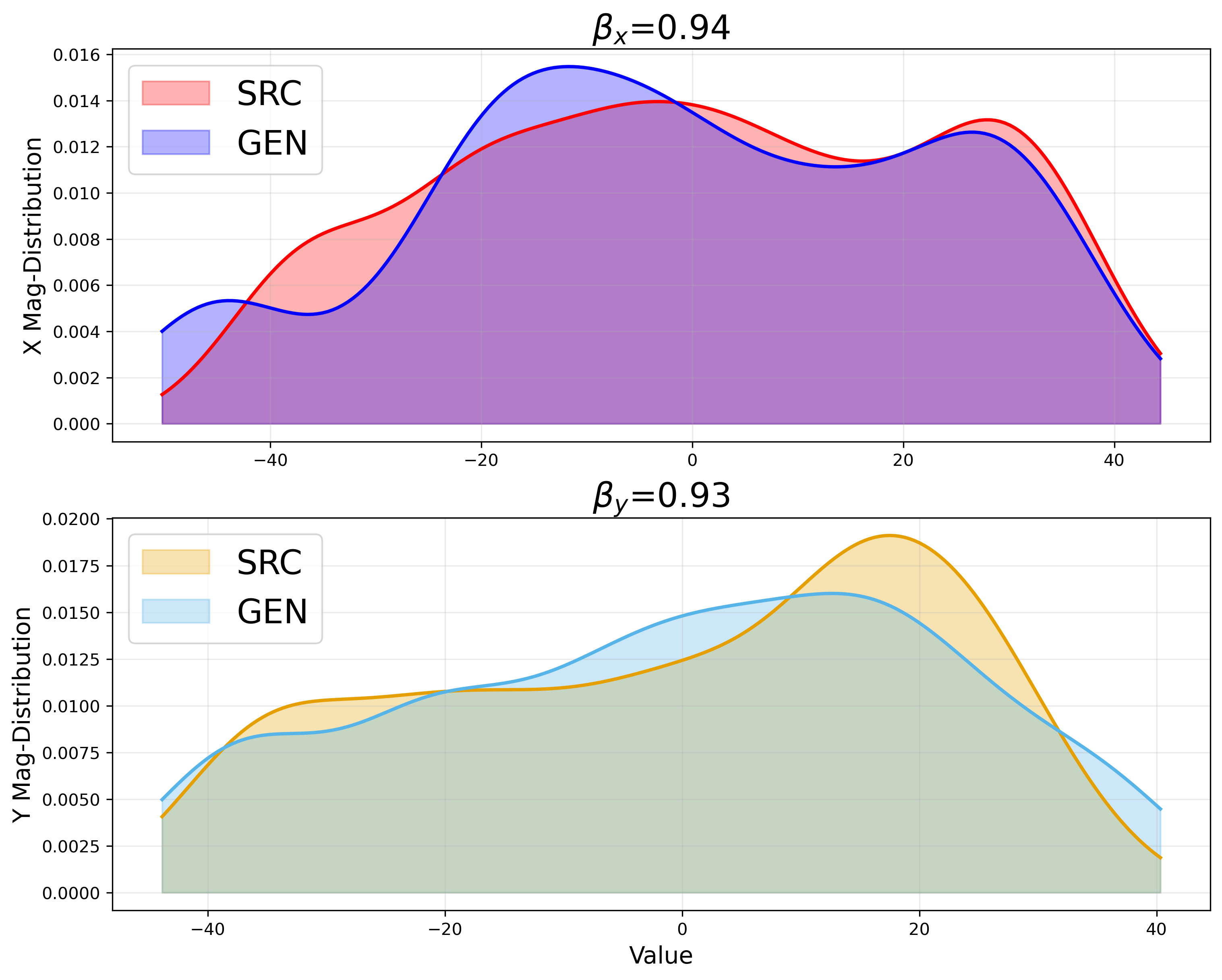}
    \caption{\revise{in editing}}
    \label{fig:subfig:t5}
\end{subfigure}
\hfill
\begin{subfigure}{0.32\linewidth}
    \centering
    \includegraphics[width=\linewidth]{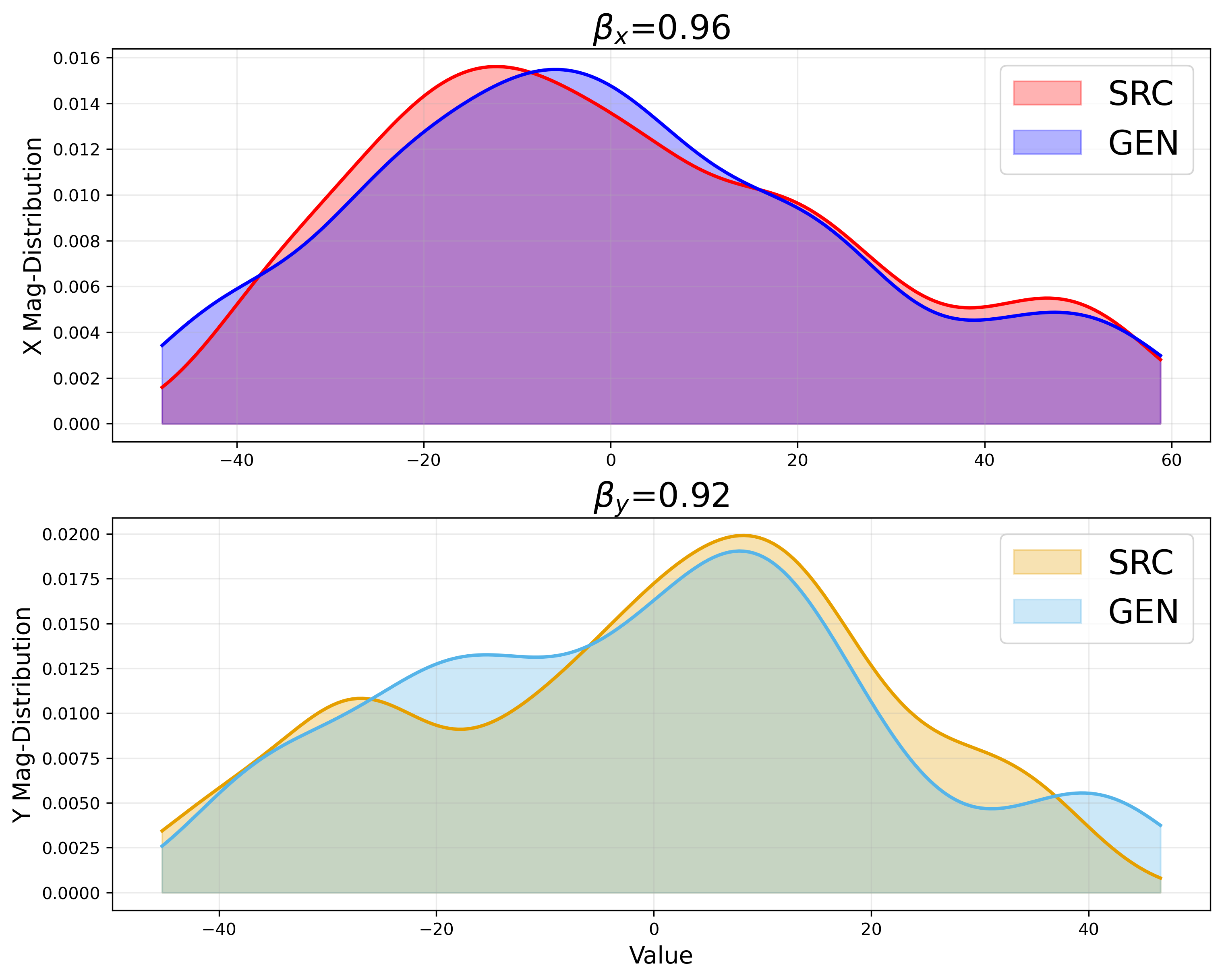}
    \caption{\revise{converged}}
    \label{fig:subfig:t6}
\end{subfigure}

\caption{\revise{The t-SNE distributions of the latent representations on original prompts (SRC) and rephrase prompts (GEN) in MLLM. The curves depict the marginal distributions along the two dimensions, with $\beta_x$ and $\beta_y$ representing the proportion of the overlap. (a)-(c) denotes the MEND, (d)-(f) denotes the MEND+\modelname.}}
\label{fig:kde_new}
\vspace{-0.3in}
\end{figure}

\nosection{Effects of Edit Trajectory Invariant Learning}
We ablate the effect of the TV-$\ell_1$ penalty strength ($\lambda$) in IRM-TV optimization, and present results in Figure \ref{fig:ablation_ood}, from which we find:
1) An insufficient penalty, \ie $\lambda=0.0001$, fails to extract feature invariance, leading to suboptimal performance.
2) An excessive penalty, \ie $\lambda=0.01$, over-constrains the model, simultaneously degrading three metrics.
%
3) The dynamic and adaptive formulation of $\lambda(\pi, \phi_e)$ shows its superiority, validating it robustly balances knowledge assimilation and discrimination across diverse environments.

\nosection{Visualization on OOD Generalization}
\revise{
We visualize the latent representations of original and rephrased prompts in MLLM with t-SNE \citep{van2008visualizing} across different editing stages.
From Figure \ref{fig:kde_new}:
1) Before editing, rephrased prompts align with original prompt distributions in the pre-trained MLLM.
2) During editing, MEND induces a marked distribution shift as $\beta_x$ and $\beta_y$ values drop, fails to extract semantic invariance. But MEND+\modelname~maintains strong alignment with high $\beta$ values, showing stable trajectory learning.
3) At convergence, the distribution shift in MEND persists while \modelname~sustains robust alignment, proving generalization to semantic-neighboring regions.
}

\nosection{Hyperparameter Sensitivity}
We study effects of the learning rate and layer depth in the IRM-TV network.
From Figure \ref{fig:para_sen}:
1) A small learning rate hinders extraction of invariant features, while a moderate increase enhances generality, accompanied by a slight sacrifice in locality.
However, an excessively large rate suppresses overall performance.
2) Deeper networks facilitate diverse cross-modal association learning, but the marginal benefit diminishes once layer depth reaches a certain level.
%

\begin{figure}
\vspace{-0.05in}
    \centering
    
    \begin{subfigure}{0.49\linewidth}
        \centering
        \includegraphics[width=0.9\linewidth]{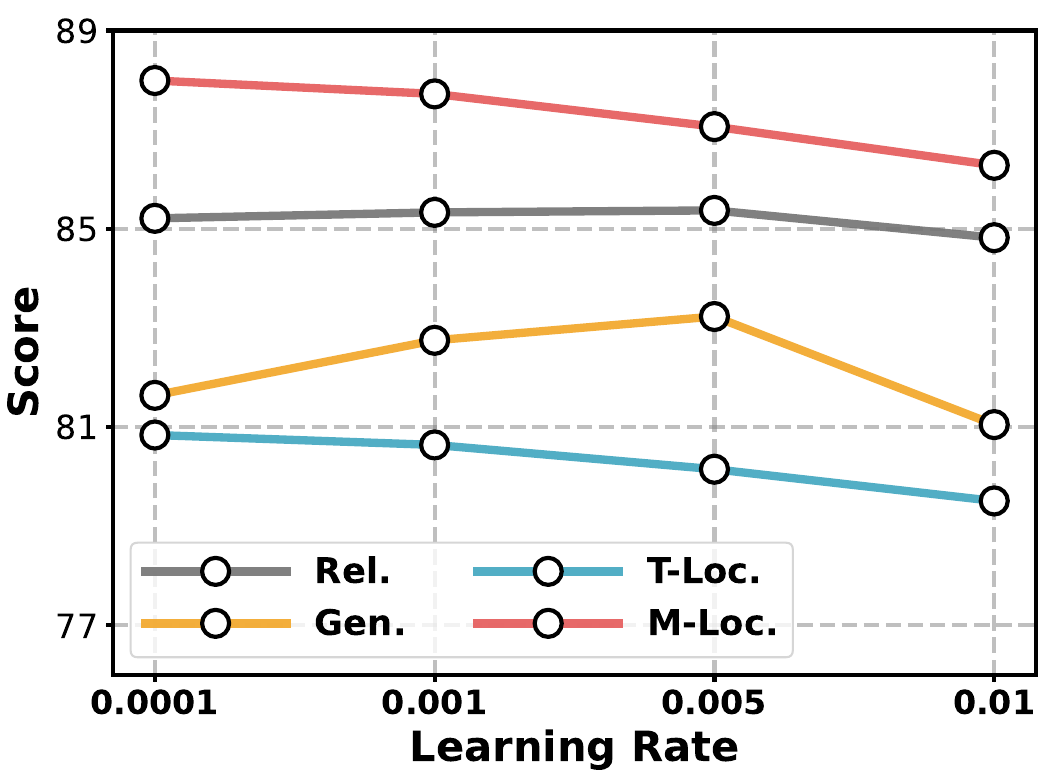}
        \label{fig:subfig:ab-1}
    \end{subfigure}
    \hfill
    \begin{subfigure}{0.49\linewidth}
        \centering
        \includegraphics[width=0.9\linewidth]{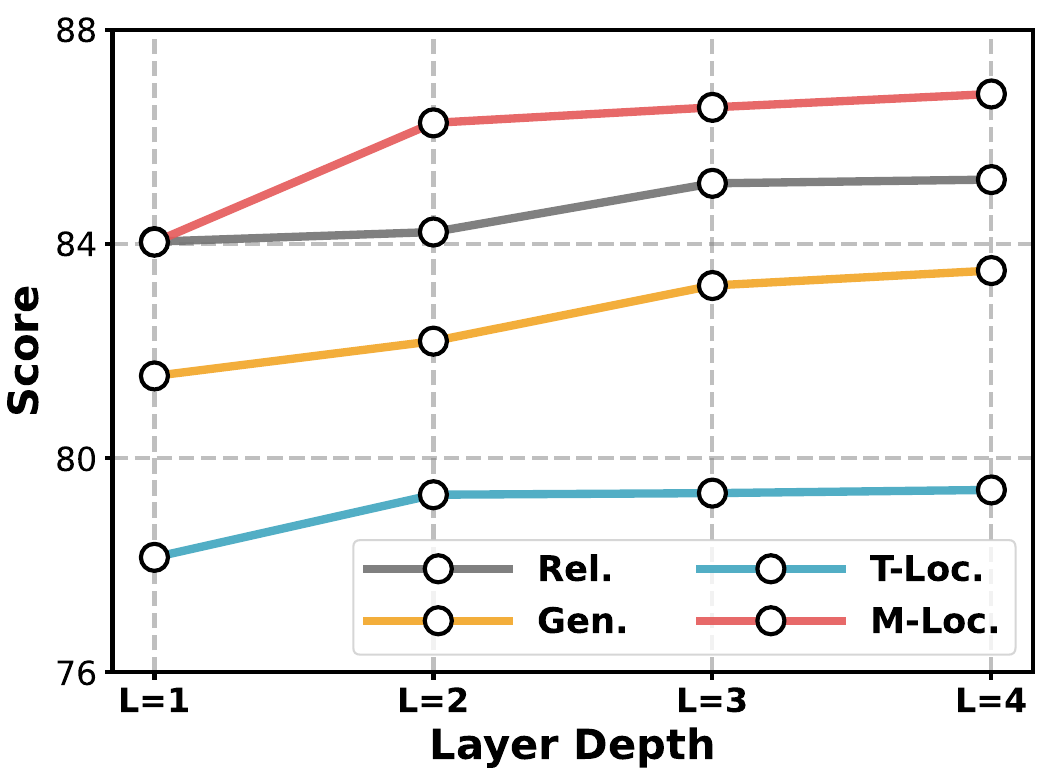}
        \label{fig:subfig:ab-2}
    \end{subfigure}
    \vspace{-0.05in}
    \caption{Effects of learning rate and layer depth.}
    \label{fig:para_sen}
\vspace{-0.3in}
\end{figure}

\nosection{Computational Cost}
\revise{We pick one typical parameter-adjusting (MEND) and model-extending (WISE) baselines for comparison.
From \ref{tab:compute-cost}:
1) \modelname~introduces a supplementary network that processes parameters from the knowledge-editing layer, leading to increased memory usage. But this is the reasonable trade-off for the gains in performance and is acceptable given the current state of computational resources.
2) \modelname~does not incur a significant increase in time cost, indicating its efficiency.
3) While integrating \modelname~increases steps, the resultant increase in total time cost does not constitute an order-of-magnitude change and remains within a practical range for real-world deployment.
Employing higher-performance computing resources can reduce this training time gap.
}

\section{Conclusion and Future Work}
In this work, we rethink knowledge editing in MLLM as an OOD generalization problem.
To identify semantic-shift and factual-shift among cross-modal prompting environments, we propose a plug-and-play invariant learning based optimization paradigm with tripartite OOD risks to jointly enhance editing reliability, locality, and generality.
This work marks an initial step in solving multimodal editing from an OOD perspective, for which we introduce simple yet general editing invariant risks with a pathway to guide robust model adaptation.
In future, researchers could investigate advanced strategies to strengthen MLLM's grasp of invariant trajectories and discern spurious factors, with refined regularization functions for more robust editing.



\nocite{langley00}

\bibliography{example_paper}
\bibliographystyle{icml2026}

\newpage
\appendix
\onecolumn
\section{Proofs}
\subsection{Equivalence between OOD-w and IRM Formulation} \label{app:ood-w}
In this section, we provide a detailed proof of the equivalence between the original out-of-distribution (OOD) editing objective and its reformulation using an environment-aware classifier $\omega$.
Specifically, we aim to show that for any edited model parameters $\phi_e$, the worst-case risk over all environments $e\in \mathcal{E}$ can be equivalently expressed as the worst-case risk over all possible classifiers $\omega \in \varSigma$, under the condition that $\omega$ captures the environmental variability through a surjective mapping.

\textbf{Proof.}  
To establish the equality, we demonstrate two inequalities.
First, we prove that 
\[
\max_{e \in \mathcal{E}} \mathcal{R}_{\text{edit}}(\omega(e) \circ \phi_e) \geq \max_{\omega \in \varSigma} \mathcal{R}_{\text{edit}}(\omega \circ \phi_e).
\]  
Let $\omega^*$ be a classifier that attains the maximum on the right-hand side, so that  
\[
\omega^* = \arg \max_{\omega \in \varSigma} \mathcal{R}_{\text{edit}}(\omega \circ \phi_e).
\]  
Given the surjectivity of the mapping $e \mapsto \omega(e)$, there exists an environment $e_0 \in \mathcal{E}$ such that $\omega(e_0) = \omega^*$. Consequently,  
\[
\mathcal{R}_{\text{edit}}(\omega(e_0) \circ \phi_e) = \mathcal{R}_{\text{edit}}(\omega^* \circ \phi_e) = \max_{\omega \in \varSigma} \mathcal{R}_{\text{edit}}(\omega \circ \phi_e).
\]  
Since $e_0$ is an element of $\mathcal{E}$, the maximum over $\mathcal{E}$ must be at least as large as the value at $e_0$, yielding the desired inequality.
Second, we prove the opposite inequality:
\[
\max_{e \in \mathcal{E}} \mathcal{R}_{\text{edit}}(\omega(e) \circ \phi_e) \leq \max_{\omega \in \varSigma} \mathcal{R}_{\text{edit}}(\omega \circ \phi_e).
\]  
Let $e^*$ be an environment that achieves the maximum on the left-hand side, i.e.,  
\[
e^* = \arg \max_{e \in \mathcal{E}} \mathcal{R}_{\text{edit}}(\omega(e) \circ \phi_e).
\]  
Then, $\omega(e^*)$ is by construction a member of $\varSigma$. Therefore,  
\[
\mathcal{R}_{\text{edit}}(\omega(e^*) \circ \phi_e) \leq \max_{\omega \in \varSigma} \mathcal{R}_{\text{edit}}(\omega \circ \phi_e),
\]  
which directly implies the inequality. By combining both inequalities, we conclude that  
\[
\max_{e \in \mathcal{E}} \mathcal{R}_{\text{edit}}(\omega(e) \circ \phi_e) = \max_{\omega \in \varSigma} \mathcal{R}_{\text{edit}}(\omega \circ \phi_e),
\]  
which holds for any $\phi_e$. Thus, minimizing either expression with respect to $\phi_e$ leads to the same optimal solution, confirming the equivalence between the OOD-$\omega$ and IRM formulations. This result allows us to leverage the IRM framework for invariant learning in multimodal knowledge editing.

\subsection{How IRM with TV-$l_1$ Penalty Achieves Editing OOD} \label{app:irm-tv-ood}

In this section, we provide a theoretical analysis demonstrating how the proposed IRM formulation with TV-$\ell_1$ penalty achieves the out-of-distribution (OOD) editing objective defined in Eq.\ref{eq:tv-ood} of the main text. Specifically, we prove that when the penalty parameter $\lambda_{\phi_e}$ is allowed to vary with the model parameters $\phi_e$, the IRM-TV objective can achieve the same optimum as the original OOD editing objective.
Recall the IRM-TV formulation from Eq.\ref{eq:tv-ood}:
\[
\min_{\phi_e} \left\{ \mathbb{E}_{\omega}[\mathcal{R}_{\text{rel}}(\omega\circ\phi_e) + \mathcal{R}_{\text{loc}}(\omega\circ\phi_e)] + \lambda_{\phi_e} \left( \mathbb{E}_{\omega}[|\nabla_{\omega}\mathcal{R}_{\text{gen}}(\omega\circ\phi_e)|] \right)^2 \right\},
\]
where the first term represents the base editing risk (reliability + locality) and the second term is the TV-$\ell_1$ penalty on the generalization risk. The original OOD editing objective is:
\[
\min_{\phi_e} \max_{\omega \in \varSigma} \mathcal{R}_{\text{edit}}(\omega\circ\phi_e).
\]

We first demonstrate through a counterexample that a fixed $\lambda$ cannot achieve the OOD objective, then prove the existence of a $\lambda_{\phi_e}$ that varies with $\phi_e$ to achieve equivalence.

\subsubsection{ The Necessity of $\lambda$ Varying with $\phi_e$ }

Following \citep{lai2024invariant}, we provide a counterexample with a fixed $\lambda$ to prove the necessity of $\lambda$ Varying with $\phi_e$.
To show that a fixed $\lambda$ is insufficient, consider a simplified editing scenario where we aim to optimize the feature parameter $\phi_e \in [-1, 1]$. The classifier $\omega$ follows a uniform distribution on $[-0.9, 0.1]$, reflecting environmental variations. The editing risk is defined as:
\[
\mathcal{R}_{\text{edit}}(\omega\circ\phi_e) := |\omega \cdot \phi_e + 1|.
\]

For this setup, the OOD objective achieves its minimum at $\phi_e = 0$ with value 1:
\begin{equation}
\begin{aligned}
\min_{\phi_e \in [-1,1]} \max_{\omega \in [-0.9,0.1]} \mathcal{R}_{\text{edit}}(\omega\circ\phi_e) = 1, \\
\arg\min_{\phi_e \in [-1,1]} \max_{\omega \in [-0.9,0.1]} \mathcal{R}_{\text{edit}}(\omega\circ\phi_e) = 0.
\nonumber
\end{aligned}
\end{equation}

However, for any fixed $\lambda \geq 0$, the IRM-TV objective:
\[
\mathbb{E}_{\omega}[\mathcal{R}_{\text{edit}}(\omega\circ\phi_e)] + \lambda \left( \mathbb{E}_{\omega}[|\nabla_{\omega}\mathcal{R}_{\text{edit}}(\omega\circ\phi_e)|] \right)^2,
\]
fails to achieve the same optimum as the OOD objective. To demonstrate this, we analyze the behavior of both objectives for the simplified editing scenario.
For $\phi_e \geq 0$, the OOD objective becomes $\max_{\omega \in [-0.9,0.1]} \mathcal{R}_{\text{edit}}(\omega\circ\phi_e) = 1 + 0.1\phi_e$,
which is minimized at $\phi_e = 0$ with value 1.
For $\phi_e < 0$, the OOD objective becomes 
$\max_{\omega \in [-0.9,0.1]} \mathcal{R}_{\text{edit}}(\omega\circ\phi_e) = 1 - 0.9\phi_e$, which is also minimized at $\phi_e = 0$ with value 1.
Now, evaluating the IRM-TV objective with fixed $\lambda$:
\[
\mathbb{E}_{\omega}[\mathcal{R}_{\text{edit}}(\omega\circ\phi_e)] = \int_{-0.9}^{0.1} (1 + \omega\phi_e) d\nu = 1 + \phi_e \cdot \mathbb{E}_{\omega}[\omega] = 1 - 0.4\phi_e,
\]
\[
\mathbb{E}_{\omega}[|\nabla_{\omega}\mathcal{R}_{\text{edit}}(\omega\circ\phi_e)|] = |\phi_e| \cdot \int_{-0.9}^{0.1} d\nu = |\phi_e|.
\]

Thus, the IRM-TV objective becomes $1 - 0.4\phi_e + \lambda\phi_e^2.$
Minimizing this quadratic function over $\phi_e \in [-1,1]$ yields:
1) If $\lambda > 0.2$, the minimum occurs at $\phi_e = 0.2/\lambda$ with value $1 - 0.04/\lambda$.
2) If $0 < \lambda \leq 0.2$, the minimum occurs at $\phi_e = 1$ with value $0.6 + \lambda$.
3) If $\lambda = 0$, the minimum occurs at $\phi_e = 1$ with value $0.6$.
Comparing with the OOD optimum ($\phi_e = 0$, value 1), for any fixed $\lambda \geq 0$ we observe that:
\[
\min_{\phi_e \in [-1,1]} \left\{1 - 0.4\phi_e + \lambda\phi_e^2\right\} \neq 1,
\]
\[
\arg\min_{\phi_e \in [-1,1]} \left\{1 - 0.4\phi_e + \lambda\phi_e^2\right\} \neq 0.
\]

This deviation occurs because the expectation term $\mathbb{E}_{\omega}[\mathcal{R}_{\text{edit}}(\omega\circ\phi_e)]$ pulls the optimum away from $\phi_e = 0$ to reduce the average risk, while the fixed $\lambda$ cannot adequately compensate for this bias. Only when $\lambda$ is allowed to vary with $\phi_e$ can we achieve equivalence with the OOD objective.

\subsubsection{Proofs on Existence of $\lambda_{\phi_e}$}

We now prove that there exists a $\lambda_{\phi_e}$ that varies with $\phi_e$ such that the IRM-TV objective equals the OOD objective for each $\phi_e$.
For the case where $\mathbb{E}_{\omega}[|\nabla_{\omega}\mathcal{R}_{\text{edit}}(\omega\circ\phi_e)|] = 0$, indicating constant generalization risk, $\lambda_{\phi_e}$ can be chosen arbitrarily since the TV term vanishes.
For the nontrivial case where $\mathbb{E}_{\omega}[|\nabla_{\omega}\mathcal{R}_{\text{gen}}(\omega\circ\phi_e)|] > 0$, we construct $\lambda_{\phi_e}$ as:
\[
\lambda_{\phi_e} := \frac{ \max_{\omega \in \varSigma} \mathcal{R}_{\text{edit}}(\omega\circ\phi_e) - \mathbb{E}_{\omega}[\mathcal{R}_{\text{rel}}(\omega\circ\phi_e) + \mathcal{R}_{\text{loc}}(\omega\circ\phi_e)+\mathcal{R}_{\text{gen}}(\omega\circ\phi_e)] }{ \left( \mathbb{E}_{\omega}[|\nabla_{\omega}\mathcal{R}_{\text{edit}}(\omega\circ\phi_e)|] \right)^2 }.
\]

This construction ensures that for each $\phi_e$, we have
\[
\mathbb{E}_{\omega}[\mathcal{R}_{\text{edit}}(\omega\circ\phi_e)] + \lambda_{\phi_e} \left( \mathbb{E}_{\omega}[|\nabla_{\omega}\mathcal{R}_{\text{edit}}(\omega\circ\phi_e)|] \right)^2 = \max_{\omega \in \varSigma} \mathcal{R}_{\text{edit}}(\omega\circ\phi_e),
\]
since the numerator represents the gap between the worst-case risk and the expected base risk.

\subsubsection{Achieving OOD-$\omega$ Optimality}

Let $\phi_e^*$ be an optimal solution of the IRM-TV objective with $\lambda_{\phi_e}$ defined above. Then for any $\phi_e$:
\begin{equation}
\begin{aligned}
\mathbb{E}_{\omega}[\mathcal{R}_{\text{rel}}(\omega\circ\phi_e^*) + \mathcal{R}_{\text{loc}}(\omega\circ\phi_e^*)+\mathcal{R}_{\text{gen}}(\omega\circ\phi_e^*)] + \lambda_{\phi_e^*} \left( \mathbb{E}_{\omega}[|\nabla_{\omega}\mathcal{R}_{\text{edit}}(\omega\circ\phi_e^*)|] \right)^2  \\
\leq \mathbb{E}_{\omega}[\mathcal{R}_{\text{rel}}(\omega\circ\phi_e) + \mathcal{R}_{\text{loc}}(\omega\circ\phi_e)+\mathcal{R}_{\text{gen}}(\omega\circ\phi_e)] + \lambda_{\phi_e} \left( \mathbb{E}_{\omega}[|\nabla_{\omega}\mathcal{R}_{\text{edit}}(\omega\circ\phi_e)|] \right)^2.
\end{aligned}
\nonumber
\end{equation}

Substituting the definition of $\lambda_{\phi_e}$, for all $\phi_e$ we have:
\[
\max_{\omega \in \varSigma} \mathcal{R}_{\text{edit}}(\omega\circ\phi_e^*) \leq \max_{\omega \in \varSigma} \mathcal{R}_{\text{edit}}(\omega\circ\phi_e).
\]
This is the proof that $\phi_e^*$ is also optimal for the OOD objective.
Conversely, if $\phi_e^*$ is optimal for the OOD objective, then for all $\phi_e$:
\[
\mathbb{E}_{\omega}[\mathcal{R}_{\text{edit}}(\omega\circ\phi_e^*) ] + \lambda_{\phi_e^*} \left( \mathbb{E}_{\omega}[|\nabla_{\omega}\mathcal{R}_{\text{edit}}(\omega\circ\phi_e^*)|] \right)^2 = \max_{\omega \in \varSigma} \mathcal{R}_{\text{edit}}(\omega\circ\phi_e^*) \leq \max_{\omega \in \varSigma} \mathcal{R}_{\text{edit}}(\omega\circ\phi_e),
\]
which shows that $\phi_e^*$ is also optimal for the IRM-TV objective.
This completes the proof that the IRM formulation with TV-$\ell_1$ penalty can achieve the OOD editing objective when $\lambda_{\phi_e}$ is properly chosen as a function of $\phi_e$.

\subsection{Computation of Gradients in Primal-dual Optimization} \label{app:gradients}

In this section, we provide the detailed computation process of the gradient $\nabla_{\delta} \mathcal{G}$ and the subgradient $\partial_{\phi_e} \mathcal{G}$ for the primal-dual optimization problem defined in Eq. \ref{eq:13} and \ref{eq:14} of the main text. Recall the Lagrangian function as:
\[
\mathcal{G}(\delta, \phi_e) = \mathbb{E}_{\omega}[\mathcal{R}_{\text{edit}}(\omega \circ \phi_e)] + \lambda(\delta, \phi_e) \left( \mathbb{E}_{\omega}[|\nabla_{\omega} \mathcal{R}_{\text{edit}}(\omega \circ \phi_e)|] \right)^2,
\]

where $\mathcal{R}_{\text{edit}}(\omega \circ \phi_e) = \mathcal{R}_{\text{rel}}(\omega \circ \phi_e) + \mathcal{R}_{\text{loc}}(\omega \circ \phi_e) + \mathcal{R}_{\text{gen}}(\omega \circ \phi_e)$ represents the complete editing risk. To compute the gradients, we assume that the risk functions are Lipschitz continuous and admit subgradients at non-differentiable points.

\nosection{Subgradient of $\mathcal{G}$ with Respect to $\phi_e$}
The subgradient $\partial_{\phi_e} \mathcal{G}(\delta, \phi_e)$ is computed as:
\begin{equation}
\begin{aligned}
\partial_{\phi_e} \mathcal{G}(\delta, \phi_e) = &\mathbb{E}_{\omega}[\nabla_{\phi_e} \mathcal{R}_{\text{edit}}(\omega \circ \phi_e)] + 2 \lambda(\delta, \phi_e) \cdot \mathbb{E}_{\omega}[|\nabla_{\omega} \mathcal{R}_{\text{edit}}(\omega \circ \phi_e)|] \\ &\cdot \mathbb{E}_{\omega}[\partial_{\phi_e} |\nabla_{\omega} \mathcal{R}_{\text{edit}}(\omega \circ \phi_e)|] 
+ \nabla_{\phi_e} \lambda(\delta, \phi_e) \cdot \left( \mathbb{E}_{\omega}[|\nabla_{\omega} \mathcal{R}_{\text{edit}}(\omega \circ \phi_e)|] \right)^2.
\end{aligned}
\nonumber
\end{equation}

Here, the term $\partial_{\phi_e} |\nabla_{\omega} \mathcal{R}_{\text{edit}}(\omega \circ \phi_e)|$ requires special handling due to the absolute value function. Based on derivations in \citep{wang2025out}, we obtain its subgradient as:

\[
\partial_{\phi_e} |\nabla_{\omega} \mathcal{R}_{\text{edit}}(\omega \circ \phi_e)| = 
\begin{cases} 
\text{sign}(\nabla_{\omega} \mathcal{R}_{\text{edit}}(\omega \circ \phi_e)) J_{\phi_e}^{-1}\left[\nabla_{\omega} \mathcal{R}_{\text{edit}}(\omega \circ \phi_e)\right]
& \text{if } \nabla_{\omega} \mathcal{R}_{\text{edit}}(\omega \circ \phi_e) \neq 0, \\
0 & \text{if } \nabla_{\omega} \mathcal{R}_{\text{edit}}(\omega \circ \phi_e) = 0,
\end{cases}
\]

where $J_{\phi_e}[\cdot]$ denotes the Jacobian matrix with respect to $\phi_e$. This formulation ensures that the subgradient is well-defined even at points where the gradient is zero.

\nosection{Gradient of $\mathcal{G}$ with Respect to $\delta$}
The gradient $\nabla_{\delta} \mathcal{G}(\delta, \phi_e)$ is computed as:
\[
\nabla_{\delta} \mathcal{G}(\delta, \phi_e) = \nabla_{\delta} \lambda(\delta, \phi_e) \cdot \left( \mathbb{E}_{\omega}[|\nabla_{\omega} \mathcal{R}_{\text{edit}}(\omega \circ \phi_e)|] \right)^2.
\]

The first term in $\mathcal{G}$, $\mathbb{E}_{\omega}[\mathcal{R}_{\text{edit}}(\omega \circ \phi_e)]$, does not depend on $\delta$, so its gradient with respect to $\delta$ is zero.

\nosection{Implementation Notes}
In practice, the expectations over $\omega$ are approximated using Monte Carlo sampling from the environment distribution. The gradients $\nabla_{\phi_e} \mathcal{R}_{\text{edit}}$, $\nabla_{\omega} \mathcal{R}_{\text{edit}}$, and $\nabla_{\delta} \lambda$ are computed using standard backpropagation. The subgradient for the absolute value term is implemented using a conditional statement, which is supported by autograd systems. This approach ensures efficient and stable optimization during the primal-dual updates.

These gradient computations enable the iterative updates in Eq. \ref{eq:15} of the main text:

\[
\phi_e^{(k+1)} = \phi_e^{(k)} - \gamma_1^{(k)} \cdot \partial_{\phi_e} \mathcal{G}(\delta^{(k)}, \phi_e^{(k)}), \quad \delta^{(k+1)} = \delta^{(k)} + \gamma_2^{(k)} \cdot \nabla_{\delta} \mathcal{G}(\delta^{(k)}, \phi_e^{(k+1)}),
\]

leading to convergence to a solution that minimizes the OOD editing risk while maintaining the invariance properties enforced by the TV-$\ell_1$ penalty.

\section{Causal grounding analysis of cascaded reasoning in MLLM}
\subsection{Architectural Causal Structure Embedded in MLLM}
\revise{
Multimodal language models implement an unidirectional computational graph, that is: unimodal encoders → cross-modal fusion → unified semantic reasoning. This forward computation defines a structural causal ordering, for which in the  Structural Causal Model (SCM) view \citep{li2024multimodal,zhou2024mitigating}, modules are equal to variables and the forward pass is equal to structural equations.
Thus the cascade reasoning is not a hypothesized causal model, but the deterministic functional decomposition of existing architectures.
}

\subsection{How Perturbations Propagate During MLLM Editing}
\revise{
Under this structural ordering, any local perturbation to a module $\Delta M$ or parameter $\Delta W$ necessarily propagates forward through the downstream modules and changes their internal states. Thus,  there is no one-to-one rigid mapping, \ie rigid mapping, between a specific parameter edit and the final output change, because the effect of the edit is mediated by all subsequent causal mechanisms in the network. Formally, for a structural chain
$$
h^{(unimodal)} \rightarrow h^{(align)} \rightarrow h^{(shared)} \rightarrow y
$$
A perturbation enters the output through
$$
y'=f_{shared}\left(f_{align}\left(f_{unimodal}(x ; W+\delta W)\right)\right)
$$
Thus the output shift $\Delta y$ depends not only on $\Delta W$ , but on how $\Delta W$  perturbs $h_{unimodal}$ , how this shifted representation perturbs $h_{align}$, and subsequently how the changed alignment influences the semantic reasoning module $h_{shared}$. This cascading mediation proves why treating \textit{parameter edit → output change} as a rigid mapping is fundamentally inaccurate in MLLMs, \ie cause \textit{casual underfit} and \textit{casual overfit} in Section Introduction.
}

\section{Definitions of Semantic Shift and Factual Shift} \label{app:shiftdefine}
The definitions of \textit{Semantic Shift} and \textit{Factual Shift} rely on three shared mappings:

\nosection{Semantic neighborhood}
 Let $f(x)$ be the MLLM’s semantic embedding. We define meaning-preserving variation via the semantic neighborhood:
$$
\mathcal{N}_{\varepsilon}(x)=\left\{x^{\prime}:\left\|f\left(x^{\prime}\right)-f(x)\right\|_2 \leq \varepsilon\right\} .
$$
\nosection{Atomic factual content}
Let $k(x)$ denote the atomic factual content (e.g., entity–attribute or entity–relation tuples). Two inputs share factual content iff $k(x)=k(x')$.

\nosection{Output-relevant concept mapping}
Let $c(x)$ denote the minimal set of conceptual factors that feed into the MLLM’s forward causal chain (perception → alignment → semantic reasoning) and determine the final output:
$$
y = MLLM(c(x)).
$$

\revise{
\begin{definition}[Semantic Shift]
A sample $x'$ exhibits semantic shift w.r.t. $x$ if and only if
$$
x^{\prime} \in \mathcal{N}_{\varepsilon}(x), \quad k\left(x^{\prime}\right)=k(x), \quad c(x) \cap c\left(x^{\prime}\right) \neq \varnothing, \quad \operatorname{MLLM}\left(c\left(x^{\prime}\right)\right)=\operatorname{MLLM}(c(x))
$$
\end{definition}
That is, semantic shift refers to variations within the semantic neighborhood while preserving factual content and preserving the output-relevant conceptual factors. Typical examples include paraphrases, lexical substitutions, stylistic rewordings, and mild visual variations.
}

\revise{
\begin{definition}[Factual Shift]
To be rigorous, there should be two kinds of factual shift, i.e., easy factual shift and hard factual shift:
$$
\text{Easy Factual Shift:} \quad x^{\prime} \notin \mathcal{N}_{\varepsilon}(x),  k\left(x^{\prime}\right) \neq k(x),  c(x) \cap c\left(x^{\prime}\right) = \varnothing,  \operatorname{MLLM}\left(c\left(x^{\prime}\right)\right) \neq \operatorname{MLLM}(c(x))$$
$$
\text{Hard Factual Shift:} \quad x^{\prime} \notin \mathcal{N}_{\varepsilon}(x),  k\left(x^{\prime}\right) \neq k(x),  c(x) \cap c\left(x^{\prime}\right) \neq \varnothing, \operatorname{MLLM}\left(c\left(x^{\prime}\right)\right) \neq \operatorname{MLLM}(c(x)) 
$$
\end{definition}
Thus, the factual shift corresponds to moving outside the semantic neighborhood while altering the atomic fact, which necessarily changes the model’s reasoning-relevant conceptual representation. The only difference between the two factual shifts is whether the prompts share part of the conceptual framing, \eg the same entities, question structure, or visual context.
}

\section{Experimental Setup Details}\label{app:implement}

\subsection{MLLM Backbones}

\nosection{BLIP2-OPT}
\cite{li2023blip} is a vision-language pre-training framework that leverages frozen pre-trained image encoders and large language models bridged by a lightweight Querying Transformer. Our setup uses ViT-L for the vision encoder and an unsupervised-trained OPT model with 2.7 billion parameters as the decoder-based language model.

\nosection{MiniGPT-4} 
\cite{zhu2023minigpt} is a vision-language model that integrates a frozen visual encoder with the frozen Vicuna language model built on LLaMA. The model employs a single projection layer to align visual features with Vicuna and uses the same pre-trained vision component as BLIP-2, consisting of ViT-G/14 from EVA-CLIP and a Q-Former. Our setup uses ViT-G/14 for the vision encoder and a forzen Vicuna model with 7 billion parameters as the decoder-based language model.

\subsection{Dataset Structures}
\revise{
The reliance of \modelname~on three distinct data splits ($\mathcal{D}_{IN}$, $\mathcal{D}_{SE}$,  $\mathcal{D}_{out}$) is not a new imposition but rather a formalization of the training datasets from benchmark MMEdit \citep{cheng2023can}, which is also the most popularly used benchmark in previous work \citep{pan2024towards}.
The MMEdit benchmark that we use explicitly provides data structured as triplets for each edit instance in the training datasets, i.e., the original edit sample (our $\mathcal{D}_{IN}$), semantically rephrase samples (our $\mathcal{D}_{SE}$), and unrelated samples (our $\mathcal{D}_{out}$).
For clarity, here we provide the data structure of a training instance example:
}
\begin{promptbox}
\textbf{src:} A photo of\\
\textbf{pred:} Wooden spoons and forks on a wooden table.\\
\textbf{rephrase:} Provide a brief overview of the image content.\\
\textbf{alt:} A selection of wooden kitchen tools on a counter.\\
\textbf{image:} val2014/COCO\_val2014\_000000386164.jpg\\
\textbf{image rephrase:} val2014\_image\_rephrase/COCO\_val2014\_000000386164.png\\
\textbf{loc:} Who was supported by the united states during mexican civil war?\\
\textbf{loc ans:} Benito Juárez.\\
\textbf{m\_loc:} val2014/COCO\_val2014\_000000297147.jpg\\
\textbf{m\_loc\_q:} What sport can you use this for?\\
\textbf{m\_loc\_a:} Motocross.
\end{promptbox}

\subsection{Baseline Methods} \label{app:baselines}
To thoroughly evaluate the effectiveness of our model \modelname, we compare it with four types of baselines:
(1) \textit{Naive fine-tuning}: FT directly tunes the last three layers of MLLM.
(2) \textit{Parameter-adjusting unimodal editing}: MEND \citep{mitchell2021fast}.
(3) \textit{Model-extending unimodal editing}: IKE \citep{zheng2023can}, SERAC \cite{mitchell2022memory}, T-Patcher \citep{huang2023transformer}, WISE \citep{wang2024wise}.
(4) \textit{Integrate parameter-adjusting and model-extending editing}: UniKE \citep{pan2024towards}.
\modelname~serves as a plug-and-play universal framework, capable of being seamlessly integrated into any editing model that relies on loss-based optimization.
Thus, we enhance one representative model under each type of baselines using \modelname, \ie WISE+\modelname, MEND+\modelname, T-Patcher+\modelname, UniKE+\modelname, and compare the results against the original models.

\nosection{Fine-tune (FT)}
Fine-tuning is the predominant paradigm for adapting pre-trained models to downstream tasks. As our baseline for multimodal editing, we adopt vanilla fine-tuning by updating the last three layers of the MLLM.

\nosection{In-context Knowledge Editing (IKE)}
\cite{zheng2023can} explores in-context learning (ICL) for knowledge editing in large language models. IKE designs demonstration templates, \ie copy, update, retain, and retrieves relevant facts from the training corpus to construct effective in-context demonstrations that guide LLMs in precise knowledge editing.

\nosection{SERAC}
\cite{mitchell2022memory} develops a memory-based editing framework, where edits are cached in an explicit memory and retrieved at inference. A scope classifier decides whether the input falls within memory coverage. When the input falls within memory coverage, it is augmented with the most relevant memory entry and forwarded to a counterfactual model for prediction.

\nosection{WISE}
\cite{wang2024wise} introduces a dual-parametric memory with a main memory for pretrained knowledge and a side memory for edits. A router determines which memory to access for each query. To support continual editing, WISE adopts sharding and merging mechanisms that isolate edits in different parameter subspaces and integrate them without conflicts.

\nosection{MEND}
\cite{mitchell2021fast} designs model editor networks with gradient decomposition, a scalable approach for fast post-hoc editing of large pre-trained language models. Instead of directly fine-tuning model parameters, MEND employs lightweight auxiliary networks to transform fine-tuning gradients, using a low-rank decomposition to keep the transformation tractable.
We set the last three layers of MLLM as the tuned target for this auxiliary network in our experiments.

\nosection{T-Patcher}
\cite{huang2023transformer} proposes a lightweight approach for model editing, aimed at revising transformer-based pre-trained language models without affecting overall performance. Instead of updating all parameters, Transformer-Patcher adds a small set of trainable neurons, \ie patches, to the FFN layer, and trains them with activation and memory losses to respond only to targeted inputs.

\nosection{UniKE}
\cite{pan2024towards} presents a unified framework for multimodal knowledge editing by combining intrinsic memory updates and external memory resorting. Both types of knowledge are represented as key-value memories and edited in the latent space. Contrastive learning disentangles semantic and truthfulness aspects, allowing intrinsic and external knowledge to guide each other.

\begin{figure*}
\centering
\includegraphics[width=\linewidth]{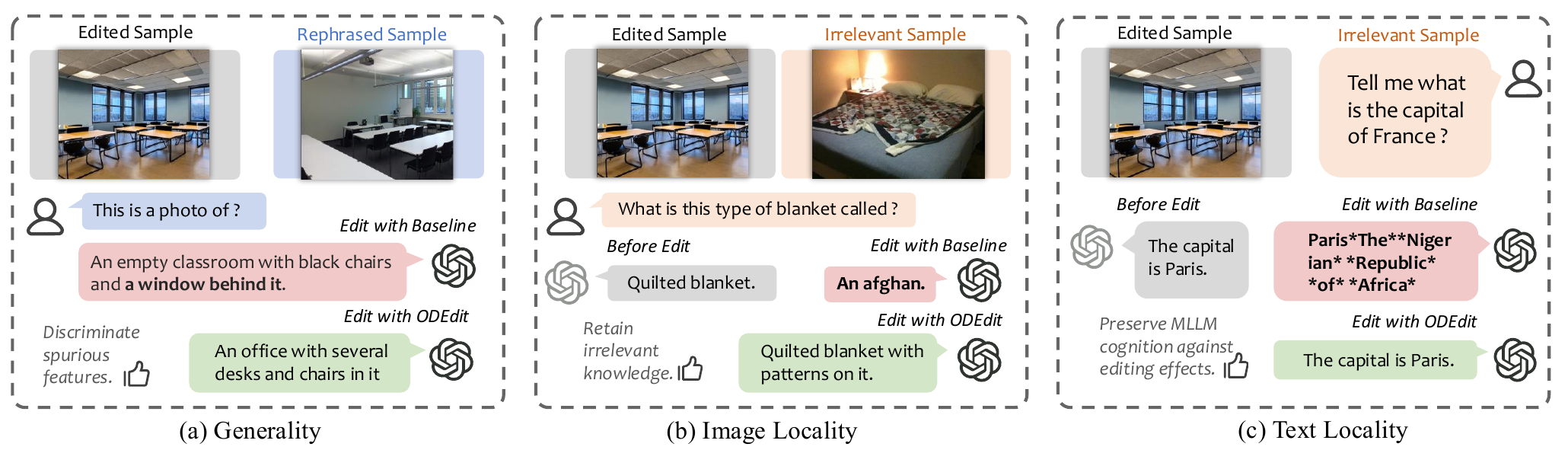}
\vspace{-0.05 in}
\caption{Case studies on the evaluation for generality, image locality, and text locality.} 
\label{fig:case}
\vspace{-0.25 in}
\end{figure*}

\subsection{Interpretability Studies} \label{app:interpret}
To evaluate the detailed effects of Maximum Mean Discrepancy Alignment and Edit Trajectory Invariant Learning, we apply the WISE method and the WISE+\modelname~ method on BLIP-2 OPT to conduct interpretability studies.
Figure \ref{fig:case} shows several qualitative cases.

For the generality evaluation, \modelname~eliminates the spurious environmental factor, \ie window, and produces generalized answers for rephrase prompts, while editing only with baseline fails to discriminate factual shifts and loses the critical invariant feature, \ie desk.
For image and text locality, \modelname~preserves accurate answers after editing, owing to the edit trajectory invariant learning.
In contrast, the cognition of MLLM on irrelevant samples is affected by editing in the baseline, leading to off-topic responses.

\end{document}